%% file: main.tex
\documentclass{article}
\usepackage[preprint]{neurips_2025}
\input{math_commands.tex}

\renewcommand{\eqref}[1]{(\ref{#1})}
\usepackage{graphicx}
\usepackage{booktabs}
\usepackage{longtable}
\usepackage{siunitx}
\usepackage{placeins}
\usepackage{needspace}
\usepackage{amsthm}
\usepackage{float}
\usepackage{algorithm}
\usepackage{algpseudocode}
\usepackage[dvipsnames]{xcolor}
\usepackage[pagebackref=true,colorlinks,citecolor=brown]{hyperref}
\usepackage{url}
\usepackage{fontawesome5}
\usepackage{enumitem}
\usepackage{caption}
\usepackage{titlesec}
\renewcommand{\paragraph}[1]{\vspace{.1em}\noindent\textbf{#1}}
\titlespacing\section{0pt}{3pt plus 1pt minus 1pt}{2pt plus 1pt minus 1pt}
\titlespacing\subsection{0pt}{2pt plus 1pt minus 1pt}{1pt plus 1pt minus 1pt}
\makeatletter
\renewcommand{\@noticestring}{}
\makeatother
\graphicspath{{figures/}{figures_v2/}}
\input{manuscript_fragments/main_facts.tex}


\input{facts_anchor_transport.tex}
\theoremstyle{plain}
\newtheorem{proposition}{Proposition}

\theoremstyle{definition}

\title{Aim Short to Reach Far: Your Frozen World Model\\ Can Plan Better Than You Think}
\author{%
\textbf{Xvyuan Liu}\textsuperscript{1,2,\textdagger},
\textbf{Jianjie Fang}\textsuperscript{1,\textdagger},
\textbf{Wei Wu}\textsuperscript{2},
\textbf{Chen Gao}\textsuperscript{1,*},
\textbf{Yong Li}\textsuperscript{1,*}\\[0.6em]
\textsuperscript{1}Tsinghua University\qquad\textsuperscript{2}Manifold AI\\[0.6em]
{\small \textsuperscript{\textdagger}Equal contribution}\\[0.2em]
{\small \textsuperscript{*}Corresponding authors: \texttt{chgao96@gmail.com, liyong07@tsinghua.edu.cn}}\\[0.2em]
{\small \href{https://daybraeklaxry.github.io/Aim-Short-to-Reach-Far/}{\textcolor{magenta}{\faGlobe\ \textit{Project Page}}}\qquad\href{https://github.com/daybraeklaxry/Aim-Short-to-Reach-Far}{\textcolor{black}{\faGithub\ \textit{Code}}}}}
\hypersetup{pdfauthor={Xvyuan Liu, Jianjie Fang, Wei Wu, Chen Gao, and Yong Li},pdftitle={Aim Short to Reach Far: Your Frozen World Model Can Plan Better Than You Think}}

\begin{document}
\raggedbottom
\widowpenalties 3 10000 10000 0
\clubpenalties 3 10000 10000 0
\maketitle
\begin{abstract}
\input{sections/00_abstract.tex}
\end{abstract}

\makeatletter
\edef\apSavedTopnumber{\the\c@topnumber}
\edef\apSavedTotalnumber{\the\c@totalnumber}
\let\apSavedTopfraction\topfraction
\let\apSavedTextfraction\textfraction
\let\apSavedFloatpagefraction\floatpagefraction
\makeatother
\setcounter{topnumber}{3}
\setcounter{totalnumber}{4}
\renewcommand{\topfraction}{0.9}
\renewcommand{\textfraction}{0.1}
\renewcommand{\floatpagefraction}{0.8}
\input{sections/01_introduction.tex}

\input{sections/03_preliminaries.tex}
\input{sections/04_method.tex}
\input{sections/05_experiments.tex}

\input{sections/02_related.tex}
\input{sections/06_conclusion.tex}
\setcounter{topnumber}{\apSavedTopnumber}
\setcounter{totalnumber}{\apSavedTotalnumber}
\let\topfraction\apSavedTopfraction
\let\textfraction\apSavedTextfraction
\let\floatpagefraction\apSavedFloatpagefraction

\Needspace{8\baselineskip}
\bibliography{refs}
\bibliographystyle{plainnat}

\clearpage
\appendix
The appendix follows the main questions: evaluation protocol and proofs first, then target comparisons and behavior (RQ1), search and goal distance (RQ2), and target construction and action selection (RQ3). The final section reports an independent-query replication. Code and per-episode results are available in our \href{https://github.com/daybraeklaxry/Aim-Short-to-Reach-Far}{repository}.

\input{sections/A0_fresh_protocol.tex}
\input{sections/A4_reproducibility.tex}
\input{sections/A1_theory.tex}

\FloatBarrier
\input{sections/C_rq1.tex}
\FloatBarrier
\input{sections/D_rq2.tex}

\input{sections/E_rq3.tex}

\FloatBarrier
\clearpage
\input{sections/F_independent.tex}
\end{document}

%% file: math_commands.tex
\usepackage{amsmath,amsfonts,bm}

\def\eqref#1{equation~\ref{#1}}

\def\1{\bm{1}}

\DeclareMathAlphabet{\mathsfit}{\encodingdefault}{\sfdefault}{m}{sl}
\SetMathAlphabet{\mathsfit}{bold}{\encodingdefault}{\sfdefault}{bx}{n}



%% file: manuscript_fragments/main_facts.tex
\newcommand{\VTwoMeanCemFinalStandard}{9.2}
\newcommand{\VTwoMeanCemLearnedStandard}{55.7}
\newcommand{\VTwoMeanCemObservedStandard}{59.8}

\newcommand{\VTwoMeanRankFinalStandard}{67.0}
\newcommand{\VTwoMeanRankLearnedStandard}{84.6}
\newcommand{\VTwoMeanRankObservedStandard}{83.8}

%% file: sections/00_abstract.tex
Latent world models plan toward goal images with a frozen pretrained predictor, without task rewards or extra trained heads. However, their planners struggle with long-range goals, and prior work addresses this by training extra components such as value functions or subgoal models. We show that the planning target itself can cause this failure: even with exact dynamics and globally optimal short-horizon search, scoring predictions by their distance to the final goal rejects the first steps of a route that initially moves away from the goal. Building on this insight, we propose \emph{Anchored Planning (AP)}, a training-free method that reuses the world model's own offline trajectories. AP retrieves a segment that leads from the current observation toward the goal and aims the frozen planner at an observation shortly after the segment's start. Across four diverse tasks, AP substantially improves frozen LeWM planners for both action synthesis and action ranking, and it outperforms both additional final-goal search and the LeWM planner on long-range goals.

%% file: sections/01_introduction.tex
\suppressfloats[t]
\input{manuscript_fragments/figure1_wrapper.tex}

\section{Introduction}
\label{sec:introduction}

\input{sections/01_intro_replacement.tex}

%% file: manuscript_fragments/figure1_wrapper.tex
\begin{figure}[t]
\centering
\includegraphics[width=\linewidth]{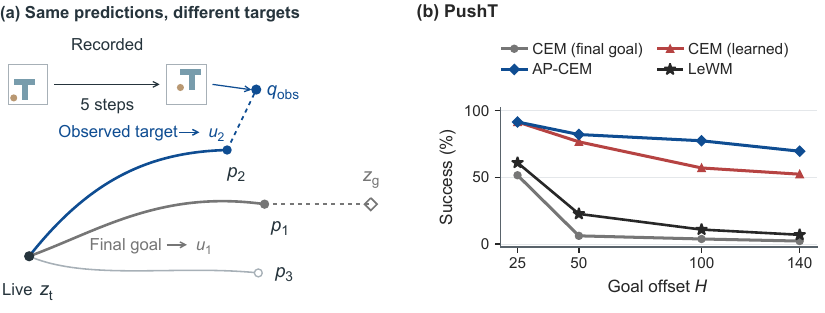}
\caption{\textbf{Changing the target improves control with the same frozen predictor.}
(a) Both target rules score the same endpoints $p_i=F(z_t,u_i)$.
Final-goal scoring selects $u_1$, whereas observed-target scoring selects $u_2$.
Solid curves show predictions, and dashed segments show scoring distances.
The target inset shows a recorded successor. Positions and observations are schematic.
(b) On PushT, intermediate targets sustain control as the recorded goal offset grows.
Curves use paired standard-start queries and the same frozen LeWM weights.}
\label{fig:capability}
\end{figure}

%% file: sections/01_intro_replacement.tex
Recently, latent world models have made it possible to plan control directly from pixels~\citep{zhou2024dinowm,assran2025vjepa2,maes2026lewm}. They encode the current and goal images and predict how the embeddings change under candidate actions. A planner then executes the short action block whose predicted embedding lies closest to the encoded goal and replans~\citep{ebert2018visual,hafner2019planet}. This recipe needs nothing beyond a goal image, and it has become a common way to turn a pretrained world model into a goal-conditioned controller~\citep{zhou2024dinowm,sobal2025pldm,terver2025jepawm,maes2026lewm}.

However, goal-distance planning breaks down as goals move farther away. \citet{caselli2026hilewm} report the decline for the LeWM planner, and final-goal CEM exhibits it on every task we study. To extend the planner's reach, prior work equips it with learned additions, for instance terminal values that lengthen a short search~\citep{hansen2022tdmpc,hansen2024tdmpc2} or high-level models that propose subgoals~\citep{zhang2026hwm,cheng2026sage}. Yet these additions change what the planner computes, so they cannot reveal whether the frozen model is too weak for distant goals or whether its planner is simply aiming at the wrong place.

In this paper, we demonstrate that the limitation can lie in where the planner aims rather than in what the model predicts. Concretely, the planner scores each prediction against a \emph{planning target}---by default the encoded goal---which objective-driven architectures treat as a cost separate from the world model~\citep{lecun2022path}. We confirm the effect of this target with a concrete control problem: the planner never moves even though every prediction is exact and each short search returns its global optimum, whereas a sequence of nearby waypoints along a feasible route carries it to the goal. The failure arises because reaching a distant goal may require first moving away from it~\citep{koren1991potential,barraquand1991rpp}. A short search scored against the final goal sees only these first steps and therefore tends to reject them, much as potential-field descent stalls at nonglobal minima~\citep{khatib1986potential,rimon1992exact}. Since the rejection comes from the score rather than from the predictions, a better model cannot remove it.

This finding motivates \emph{Anchored Planning (AP)}, which leaves the frozen model and planner untouched and replaces the final goal with a waypoint drawn from experience. At each decision, AP searches the recorded trajectories that come with the world model for a segment that begins like the current image and ends like the goal, and takes a frame near the beginning of that segment as the target. The frozen model then predicts action blocks from the current state and scores them against this waypoint, so the recorded route decides where to aim while live prediction decides how to get there. One waypoint serves two action rules: AP-CEM synthesizes actions from observation-only memory, and AP-rank ranks recorded action blocks from action-labeled memory. Because every waypoint comes from data the world model already came with, AP trains nothing new, and each evaluation episode is withheld from its memory.

With frozen LeWM models on Cube, PushT, Reacher, and TwoRoom, this change lifts success for both action rules, and the gains widen with goal distance, precisely where final-goal planning falters. Action synthesis improves even on nearby goals and further on distant ones, whereas ranking, which is near its ceiling on nearby goals, gains once the goal recedes. The recovered behavior follows the mechanism above. Across all four tasks, most successful AP-CEM rollouts take a physical detour, the kind of move that final-goal scoring penalizes, and on the two manipulation tasks most final-goal failures stall in place. Extra search is no substitute: a few CEM iterations toward an intermediate target beat many more toward the final goal, and AP surpasses LeWM's own planner at long goal offsets. Observed targets also match or exceed, in mean success, a learned target that predicts recorded successors more accurately.

Our contributions are summarized as follows:
\begingroup
\renewcommand{\labelenumi}{\textbf{\roman{enumi})}}
\setlength{\leftmargini}{1.8em}
\begin{enumerate}
    \item We identify the choice of planning target as a bottleneck of frozen world-model planning, and prove that aiming at the final goal can halt a planner whose predictions and short-horizon search are both perfect.
    \item We develop \textbf{\emph{Anchored Planning (AP)}}, which needs no training and draws a nearby target from a retrieved recorded route and uses it both to synthesize actions and to rank recorded ones with a frozen predictor.
    \item Our experiments show that retargeting lifts frozen LeWM planners beyond what a larger search budget achieves, and that observed targets rival or surpass a trained target predictor in mean success at no training cost.
\end{enumerate}
\endgroup
Our findings indicate that frozen world models hold far more planning capability than final-goal scoring allows them to express, and that aiming short is what releases it.

%% file: sections/03_preliminaries.tex
\section{Why the Planning Target Matters}
\label{sec:preliminaries}
\label{sec:half-gap}

A frozen encoder $E$ maps current and goal images to $z_t=E(o_t)$ and $z_g=E(o_g)$. An action block $u$ contains $L=5$ successive primitive actions. The frozen predictor $F(z_t,u)$ estimates the encoded state after that block and includes the fixed action normalization. The planner seeks the block whose predicted endpoint is closest to a target $q$:
\begin{equation}
\label{eq:scoring}
\hat u(q)\in\arg\min_{u\in\mathcal C}J_q(u),
\qquad J_q(u)=\|F(z_t,u)-q\|_2^2.
\end{equation}
For a finite candidate set, we score every block. For continuous actions, we approximate the optimization with the cross-entropy method (CEM)~\citep{deboer2005cem}. The controller executes the selected block and observes again. Physical success is always measured against the final goal, and $q$ only decides which block to execute next.

For two fixed predictions $p_u=F(z_t,u)$ and $p_v=F(z_t,v)$, let $\Delta_q(u,v)=J_q(u)-J_q(v)$. A negative value favors $u$ over $v$. Changing the target from $q$ to $q'$ gives
\begin{equation}
\label{eq:target-order}
\Delta_{q'}(u,v)-\Delta_q(u,v)=2(p_u-p_v)^\top(q-q').
\end{equation}
Thus the same predicted endpoints can favor different actions when the target changes, as Figure~\ref{fig:capability}a illustrates. The following construction shows how this changes control even when the short-horizon search finds the global minimum.

\begin{samepage}
\begin{proposition}[Goal-distance control can stall]
\label{prop:exact-target}
Even with smooth dynamics, bounded continuous actions, and exact predictions, globally optimal five-action final-goal scoring can select the zero block forever from a continuum of goal-reachable states. Under identical dynamics, actions, and execution, six successive endpoints of one feasible recorded trajectory instead lead every such start to the goal.
\end{proposition}
\end{samepage}

The construction places states on the curve $c(s)=(s,s^2)$, with goal $c(2)$. Each primitive action increases $s$ by an amount in $[0,1/8]$, so exact predictions follow the same continuous route toward the goal. Yet squared final-goal distance along that route is
\begin{equation}
D(s)=\|c(s)-c(2)\|^2=(s-2)^2+(s^2-4)^2.
\label{eq:curve-goal-distance}
\end{equation}
For starts $s\in[-3/2,-1]$, every nonzero admissible block ends farther from the goal under $D$, whereas the zero block leaves the distance unchanged. The global optimum of the local objective is therefore to stay still. The predictor correctly places every nonzero block farther along the route, but the distance score favors staying still.

Replacing the final goal with successive targets on the recorded curve changes this choice. Each target can be attained within the local action horizon in this construction, and selecting it advances along the same dynamics until the goal is reached. Appendix~\ref{app:theory} gives the complete construction and proof. Anchored Planning applies this idea by retrieving a task-relevant trajectory and predicting actions toward an early observed successor. Section~\ref{sec:experiments} measures the resulting target effect with frozen visual predictors.

%% file: sections/04_method.tex
\begin{figure}[t]
\centering
\includegraphics[width=\linewidth]{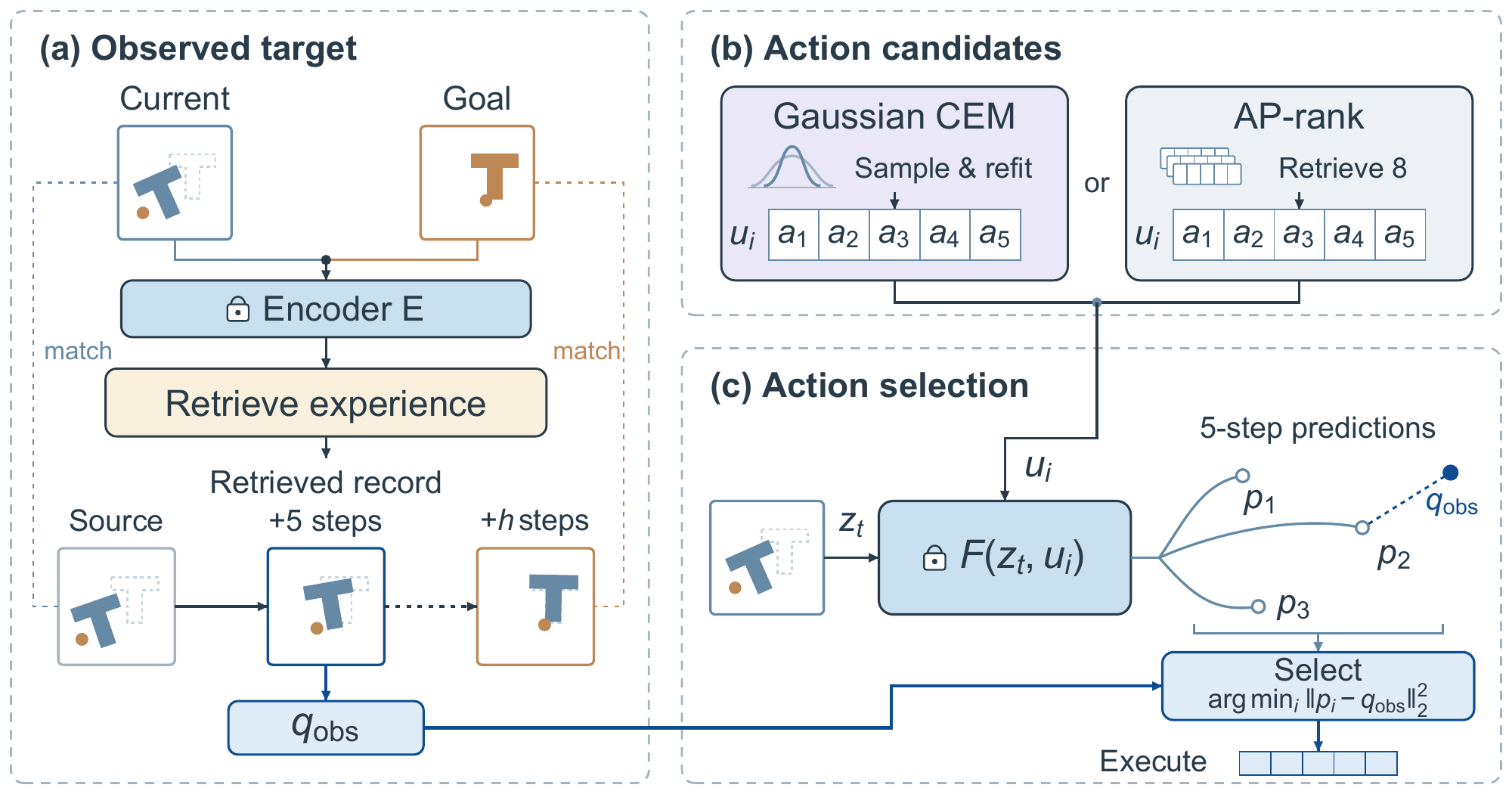}
\caption{\textbf{Anchored Planning: observed targets, live predictions.} (a) We match the current and goal images to the source and distant images of each record. The closest record's five-step successor supplies $q_{\rm obs}$. The distant offset is $h=\max(5,H-t)$, for goal offset $H$ and $t$ executed actions. (b) Each strip represents a candidate block $u_i$. Gaussian CEM synthesizes these blocks, whereas AP-rank retrieves eight recorded blocks. (c) The frozen model predicts endpoints $p_i=F(z_t,u_i)$. The dashed segment shows distance to $q_{\rm obs}$. We select the block whose endpoint is closest to the target, execute up to five actions, and observe again before replanning. All glyphs and positions are schematic.}
\label{fig:architecture}
\end{figure}

\section{Anchored Planning}
\label{sec:anchoring}

Figure~\ref{fig:architecture} shows how Anchored Planning separates choosing a destination from choosing an action. We take an observed target, a waypoint on a recorded route toward the goal, from experience and use the frozen model to predict how candidate actions approach it from the current state. Both memory settings below are built from the offline trajectories that accompany the world model, and no evaluation episode is ever stored in them. \emph{Observation-only memory} contains ordered images with episode and time indices but no actions, so the planner must synthesize them. \emph{Action-labeled memory} also contains the recorded actions, which the planner can rank. Both use the same frozen predictor, target rule, and fixed memory.

\subsection{Use distant experience to choose an observed target}
\label{sec:support-bank}

A short search needs a nearby destination that belongs to a route toward the final goal. We therefore match a record through its source and distant observations, then use its early successor to score actions. Each record contains a source, an $L$-step successor, and an $h$-step endpoint. We use the distant endpoint to match the longer route to the goal, while predicting only the next action block.

Each query provides the recorded action offset $H$ between its starting image and goal image. After $t$ executed primitive actions, we retrieve records spanning
\begin{equation}
\label{eq:remaining}
h(t)=\max\{L,H-t\}.
\end{equation}
As execution advances, retrieval matches progressively shorter recorded spans to the final goal, down to $L$. The observed target and prediction horizon remain $L$ steps ahead of their respective source states.

Let $s$ index a recorded start. For retrieval horizon $h$, memory $\mathcal B_h$ contains starts for which both $s+L$ and $s+h$ remain within the same episode. Each record pairs the retrieval key
\[
b_s^{(h)}=(z_s,\ z_{s+h},\ z_{s+h}-z_s)
\]
with its successor encoding $z_{s+L}$. Retrieval uses ordered observations and episode indices. For AP-rank, a record additionally provides its next-five-action block $u_s=(a_s,\ldots,a_{s+L-1})$. The current query is $b_t=(z_t,z_g,z_g-z_t)$. Retrieval ranks records by standardized squared distance,
\begin{equation}
\label{eq:retrieval}
d_h(t,s)=\left\|\frac{b_t-b_s^{(h)}}{\sigma_h}\right\|_2^2.
\end{equation}
Here $\sigma_h$ is the coordinate-wise population standard deviation over eligible records. The closest start $s_1$ supplies the target
\begin{equation}
\label{eq:observed-target}
q_{\rm obs}=z_{s_1+L}=E(o_{s_1+L}).
\end{equation}

\subsection{Synthesize actions from observation-only memory}
\label{sec:observation-search}

AP-CEM maintains a Gaussian distribution over five-action blocks. From a fixed prior, it clips proposals to the action bounds and scores their predicted endpoints against $q_{\rm obs}$. The best-scoring proposals update the distribution. We compare the clipped means before and after optimization and execute the lower-cost block. Appendix~\ref{app:fresh-protocol} gives the complete update and action mapping.

\subsection{Rank retrieved actions from the current observation}
\label{sec:selection}
\label{sec:computation}

AP-rank takes one action block from each of the eight retrieved records and predicts its endpoint from the current observation. We score all eight endpoints against the same observed target, taken from the closest record. Let $P$ clip each action coordinate to its task-specific bounds. The clipped candidate set is $\mathcal C=\{P(u_{s_1}),\ldots,P(u_{s_8})\}$, and AP-rank selects
\begin{equation}
\label{eq:anchored}
\hat u=\arg\min_{u\in\mathcal C}\|F(z_t,u)-q_{\rm obs}\|_2^2,
\end{equation}
with ties resolved by retrieval order.  After each primitive action, we check for success, termination, and the end of the action allowance. We observe and replan after at most five actions, and Algorithm~\ref{alg:anchored} in Appendix~\ref{app:reproducibility} summarizes the full loop.

\subsection{Compared targets and baselines}
\label{sec:controls}

We compare final-goal, learned, and observed targets under each action rule. Final-goal scoring sets $q=z_g$. A learned-target MLP predicts the observed successor from the current and goal encodings, their difference, and the retrieval span. It is trained on the same memory observations as AP, and its configuration is selected on held-out memory episodes, as described in Appendix~\ref{app:v2-learned-target}.

For CEM, only the target changes: it determines which predicted outcomes refine the Gaussian, while initialization, seeds, search settings, and action mapping stay fixed. Ranking scores the same eight retrieved blocks from a shared live state. Each controller then repeats its target rule and action selection along the resulting rollout.

\emph{Direct} executes the closest record's clipped action block, then observes and retrieves again. Comparing Direct with predictive ranking measures the benefit of evaluating recorded actions from the current state.

%% file: sections/05_experiments.tex
\section{Experiments}
\label{sec:experiments}

We organize the experiments around three questions about the effect of target choice, the role of search, and the design of useful targets:\\
\textbf{RQ1:} How much control does changing the target reveal?\\
\textbf{RQ2:} Can more search or longer lookahead recover the same gains?\\
\textbf{RQ3:} What makes a target useful, and how does prediction help reach it?

\subsection{Evaluation setup}
\label{sec:setup}

\paragraph{Models, memory, and paired queries.}
We evaluate Cube, PushT, Reacher, and TwoRoom with task-specific pretrained LeWM models~\citep{maes2026lewm}. Encoders, predictors, and the original action normalizers remain frozen. Memory contains the supplied training trajectories. Each task has 128 query episodes. We exclude them from memory and never draw perturbation prefixes from them. Each query pairs a start image with the image $H$ actions later in the same episode. We call these the original offsets, which span 100--150 actions; Table~\ref{tab:setup} lists them per task. All experiments below use these paired queries. Appendix~\ref{app:independent} repeats the target comparison on independent queries.

\paragraph{Starts and success.}
Recorded actions can have different effects from a displaced state. We therefore evaluate each recorded (standard) start and two perturbed starts reached by executing preassigned five-action prefixes. The goal, memory, and subsequent allowance stay fixed. Controllers receive images, $H$, and the allowance. We use physical states to restore paired starts and measure success after every primitive action. For perturbed starts, we average the two starts of each query, and task means weight all tasks equally. Appendix~\ref{app:fresh-protocol} gives the evaluation details.

\paragraph{Action rules and targets.}
Single-block CEM runs 30 iterations with 300 candidates and 30 elites, and ranking chooses among eight retrieved blocks. Both use the final, learned, or observed targets defined in Section~\ref{sec:controls}. They replan after at most five actions and clip proposals to the action bounds before prediction and execution. The LeWM planner keeps its action mapping and plans and executes 25 actions at a time. On shared queries, our implementation and the LeWM evaluator agree on nearly all episode outcomes, as Appendix~\ref{app:lewm-audit} shows.

\input{manuscript_fragments/main_table_wrapper.tex}

\subsection{RQ1: How much control does changing the target reveal?}
\label{sec:gaussian-target-effect}

\paragraph{Both intermediate targets improve control.}
Table~\ref{tab:main-results} shows that changing the target while retaining the frozen predictor and action rule improves success on every task at standard and perturbed starts. At standard starts, CEM's mean rises from \VTwoMeanCemFinalStandard\% with the final goal to \VTwoMeanCemLearnedStandard\% with a learned target and \VTwoMeanCemObservedStandard\% with an observed target. Ranking rises from \VTwoMeanRankFinalStandard\% to \VTwoMeanRankLearnedStandard\% and \VTwoMeanRankObservedStandard\%, respectively. Thus both learned and observed targets let the same frozen model and action rule succeed where final-goal scoring fails. AP achieves these gains with targets drawn directly from recorded experience.

\paragraph{Target choice matters more at longer offsets.}
Figure~\ref{fig:v2-horizon} compares goal offsets while keeping starts and lookahead fixed. Final-goal CEM declines from the shortest to the original offset on all tasks. \citet{caselli2026hilewm} observe the same decline for the non-hierarchical LeWM planner on PushT, with success rates of 94.0\%, 52.7\%, and 18.0\% at offsets 25, 50, and 75 under their protocol. At standard starts, as the goal offset grows from $H=25$ to $H=100$, the mean gain from observed targets rises from 33.4 to 47.9 percentage points for CEM and from 1.2 to 15.4 points for ranking. Table~\ref{tab:horizon-full} shows the same trend at perturbed starts. When the goal is nearby, ranking with the final goal already works well, and intermediate targets matter mainly once the goal is farther away.

\begin{figure}[t]
\centering
\includegraphics[width=\linewidth]{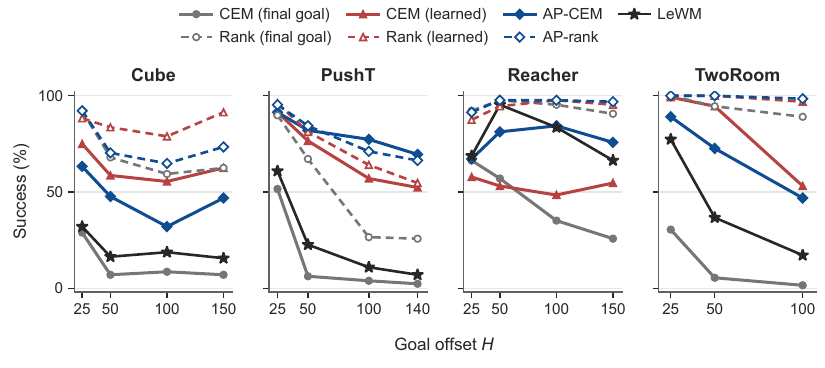}
\caption{\textbf{Observed targets give larger mean gains beyond the shortest offset.} We keep the standard starts fixed as $H$ changes and adjust the action allowance to preserve each task's $B/H$. Planner lookahead stays fixed. Results at the original offsets are those in Table~\ref{tab:main-results}. Table~\ref{tab:horizon-full} gives perturbed results.}
\label{fig:v2-horizon}
\end{figure}

\paragraph{The recovered control includes physical detours.}
We normalize physical error by the task's success thresholds. A success counts as a detour when this error rises above a previous minimum by a fixed margin, and a failure counts as stalling when there is little task-relevant motion near the end of the rollout. Successful actions can initially increase physical error, as the construction in Section~\ref{sec:preliminaries} illustrates. Figure~\ref{fig:v2-mechanism} shows this pattern. On PushT at standard starts, all 89 AP-CEM successes take a detour, whereas 100 of 125 final-goal CEM failures stall. Table~\ref{tab:v2-mechanism} confirms that most AP-CEM successes take a detour on every task, and that most final-goal CEM failures stall on the two manipulation tasks, Cube and PushT.

\begin{figure}[t]
\centering
\includegraphics[width=\linewidth]{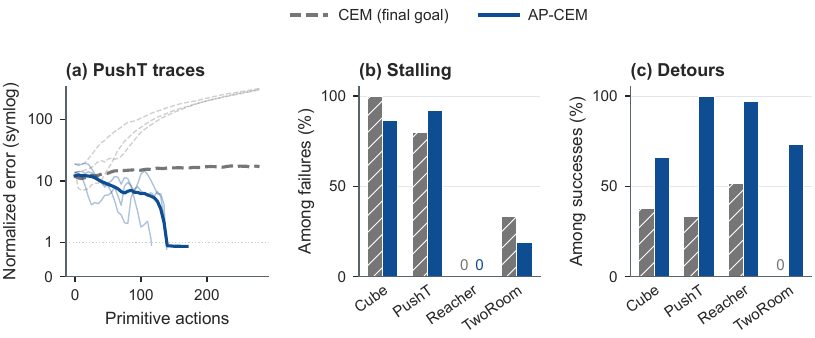}
\caption{\textbf{Successful AP-CEM control takes detours, while final-goal failures often stall.} (a) We compare final-goal CEM failures with AP-CEM successes at standard starts. Thin curves show the first three eligible queries, and thick curves show group medians with terminal error carried forward. The axis is linear below one and logarithmic above. (b) Bars show the fraction of failures that stall. (c) Bars show the fraction of successes that take a detour. Both rates use rollouts that enter control. Appendix~\ref{app:v2-physical-behavior} gives counts and threshold sensitivity.}
\label{fig:v2-mechanism}
\end{figure}

\subsection{RQ2: Can more search or longer lookahead recover the same gains?}
\label{sec:long-comparison}

\textbf{Target changes outperform more final-goal search.} RQ1 kept the search budget fixed. Here we ask whether more search or a longer lookahead can close the same gap. Figure~\ref{fig:search-budget} shows that two CEM iterations with either intermediate target already outperform 30 iterations with the final goal, on every task and at both kinds of start.

\textbf{Short observed-target search exceeds LeWM at the original offsets.} AP-CEM outperforms LeWM on every task and at both kinds of start. On Cube, PushT, and TwoRoom, it does so while predicting only 63.4--85.5\% as many action blocks per episode as LeWM. On Reacher, about 1.27 times as many blocks raise standard-start success by 9.4 points. Figure~\ref{fig:v2-horizon} shows that LeWM's 25-action plans are most competitive at shorter Reacher offsets. Once LeWM replans every five actions, as AP-CEM does, AP-CEM outperforms it on every task and at both kinds of start, as Table~\ref{tab:lewm-replan5} reports.

\begin{figure}[t]
\centering
\includegraphics[width=\linewidth]{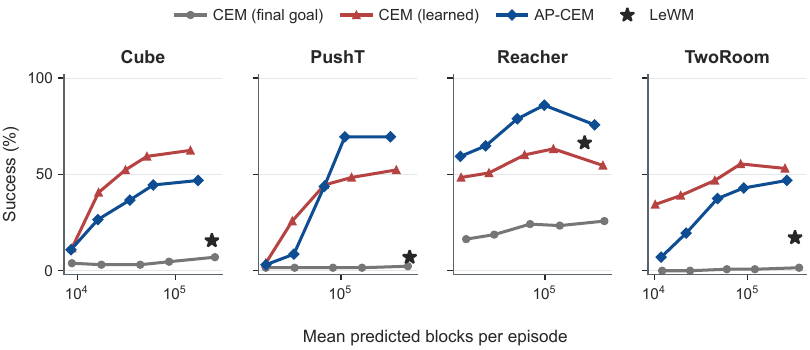}
\caption{\textbf{Two CEM iterations with an intermediate target outperform 30 with the final goal on every task.} Curves show standard-start success at iteration counts $1,2,5,10,30$, and stars mark LeWM. We measure work by the mean number of five-action blocks predicted per episode, counted until the episode ends. Table~\ref{tab:budget-full} shows the same comparison at perturbed starts.}
\label{fig:search-budget}
\end{figure}

\subsection{RQ3: What makes a target useful, and how does prediction help reach it?}
\label{sec:local-target}

\paragraph{Successor accuracy and control favor different target choices.}
The learned target is trained to match a recorded successor, whereas control is evaluated by reaching the final goal, so the two objectives can favor different targets. Table~\ref{tab:target-quality-control} shows that the learned target predicts held-out successors with lower mean error on every task, yet observed targets achieve higher mean CEM success without any training. Their largest margins arise on PushT and Reacher, where success rises from 52.3\% to 69.5\% and from 54.7\% to 75.8\%. Successor accuracy therefore does not determine how well a target guides control: an observed successor lies on a recorded route whose endpoint resembles the goal.

\input{manuscript_fragments/target_quality_control_wrapper.tex}

\paragraph{Nearby targets and retrieval timing shape success.}
Table~\ref{tab:v2-ablation} compares target offsets while keeping predictions at five actions. A three-step target raises AP-CEM's mean success to 69.9\%, while AP-rank performs best at the default five-step offset among those tested. One- and ten-step targets give lower means than the default for both rules. Ranking depends on shrinking the retrieval span as execution advances. With the span fixed at $H$, standard-start success falls from 83.8\% to 9.6\%. With the adaptive span, using ten percent of memory preserves most of the success rate. The retrieval key itself is robust: removing either its distant endpoint or its displacement term changes AP-rank's mean by at most 1.9 points, so the adaptive span, rather than fine tuning of the key, drives ranking success. Table~\ref{tab:v2-ablation}b reports these variants.

\input{manuscript_fragments/ablation_table_wrapper.tex}

\paragraph{An observed destination adjusts the requested motion.}
For recorded displacement $d=z_{s+L}-z_s$ and live offset $e=z_t-z_s$, the transported target $z_t+d$ requests the full recorded motion, whereas the observed target requests $d-e$. Anchoring the target to the observed endpoint raises mean CEM success from 53.7\% to 59.8\% at standard starts and from 46.0\% to 55.1\% after perturbation. Table~\ref{tab:v2-transport} breaks this down by task, with anchoring ahead on Cube, PushT, and TwoRoom.

\paragraph{Live prediction improves the choice of action.}
\label{sec:capability}
\label{sec:mechanism-work}
To assess how actions approach the target, we execute all eight candidates in the simulator from the same start and compare the resulting endpoints with estimates from displacement transfer and the frozen predictor, respectively:
\begin{equation}
\widetilde p_j=z_t+(z_{s_j+L}-z_{s_j}),\qquad p_j=F(z_t,u_j).
\label{eq:prediction-versus-transfer}
\end{equation}
Table~\ref{tab:confirmation-diagnostics} shows that the predictor reduces both endpoint error and selection regret relative to displacement in every task and start setting. It also reduces regret relative to Direct. We measure error as squared distance to the simulated endpoint and regret as the chosen action's exact observed-target cost above the best candidate's cost. Appendix~\ref{app:v2-exact-endpoints} gives first-block interventions followed by Direct in every rollout.

In closed-loop control, AP-rank gains most over Direct on PushT, raising success from 51.6\% to 66.4\% at standard starts and from 18.0\% to 47.7\% after perturbation. Experience specifies where to aim, and live prediction determines which recorded action approaches it from the current state.

\input{manuscript_fragments/diagnostic_table_wrapper.tex}

%% file: manuscript_fragments/main_table_wrapper.tex
\begin{table}[t]
\centering
\caption{\textbf{Intermediate targets improve control with the same frozen model.}
We report success (\%) on 128 paired queries per task. Column groups indicate
what experience each controller can use. Learned targets are trained on memory
observations. For perturbed starts, we average within each query, and task means
weight all tasks equally. Bold marks the highest success within each group of
controllers. AP-CEM and AP-rank use observed targets.}
\label{tab:main-results}
\input{tables_v2/main.tex}
\end{table}

%% file: tables_v2/main.tex
\begingroup
\small
\setlength{\tabcolsep}{2pt}
\begin{tabular*}{\linewidth}{@{\extracolsep{\fill}}llrrrrrrrr@{}}
\toprule
 &  & \multicolumn{1}{c}{No memory} & \multicolumn{3}{c}{Observation-only} & \multicolumn{4}{c}{Recorded actions} \\
\cmidrule(lr){3-3}\cmidrule(lr){4-6}\cmidrule(l){7-10}
Task & Start & \multicolumn{1}{c}{LeWM} & \multicolumn{1}{c}{\shortstack{CEM\\final}} & \multicolumn{1}{c}{\shortstack{CEM\\learned}} & \multicolumn{1}{c}{AP-CEM} & \multicolumn{1}{c}{Direct} & \multicolumn{1}{c}{\shortstack{Rank\\final}} & \multicolumn{1}{c}{\shortstack{Rank\\learned}} & \multicolumn{1}{c@{}}{AP-rank} \\
\midrule
Cube & Standard & 15.6 & 7.0 & \textbf{62.5} & 46.9 & 82.0 & 62.5 & \textbf{91.4} & 73.4 \\
 & Perturbed & 12.5 & 2.7 & \textbf{38.7} & 29.7 & 47.7 & 41.8 & \textbf{50.0} & 45.7 \\
\addlinespace[4pt]
PushT & Standard & 7.0 & 2.3 & 52.3 & \textbf{69.5} & 51.6 & 25.8 & 54.7 & \textbf{66.4} \\
 & Perturbed & 9.4 & 2.7 & 48.8 & \textbf{64.8} & 18.0 & 12.9 & 35.5 & \textbf{47.7} \\
\addlinespace[4pt]
Reacher & Standard & 66.4 & 25.8 & 54.7 & \textbf{75.8} & \textbf{99.2} & 90.6 & 95.3 & 96.9 \\
 & Perturbed & 68.8 & 22.7 & 59.0 & \textbf{78.9} & \textbf{96.9} & 87.5 & \textbf{96.9} & 95.7 \\
\addlinespace[4pt]
TwoRoom & Standard & 17.2 & 1.6 & \textbf{53.1} & 46.9 & \textbf{99.2} & 89.1 & 96.9 & 98.4 \\
 & Perturbed & 16.8 & 1.6 & \textbf{55.1} & 46.9 & \textbf{100.0} & 88.7 & 98.4 & \textbf{100.0} \\
\midrule
Mean & Standard & 26.6 & 9.2 & 55.7 & \textbf{59.8} & 83.0 & 67.0 & \textbf{84.6} & 83.8 \\
 & Perturbed & 26.9 & 7.4 & 50.4 & \textbf{55.1} & 65.6 & 57.7 & 70.2 & \textbf{72.3} \\
\bottomrule
\end{tabular*}
\endgroup

%% file: manuscript_fragments/target_quality_control_wrapper.tex
\begin{table}[t]
\centering
\captionsetup{justification=raggedright,singlelinecheck=false}
\caption{\textbf{Training-free observed targets achieve higher mean success than a more accurate learned target.}
We measure the mean squared distance to the recorded successor on held-out memory transitions and CEM success on standard-start evaluation queries.
Bold marks higher success. Appendix~\ref{app:v2-learned-target} explains how target accuracy is measured.}
\label{tab:target-quality-control}
\input{tables_v2/target_quality_control.tex}
\end{table}

%% file: tables_v2/target_quality_control.tex
\begingroup
\small
\renewcommand{\arraystretch}{1}
\setlength{\tabcolsep}{7pt}
\begin{tabular}{@{}lrrrr@{}}
\toprule
 & \multicolumn{2}{c}{Successor error} & \multicolumn{2}{c}{CEM success (\%)} \\
\cmidrule(lr){2-3}\cmidrule(l){4-5}
Task & \multicolumn{1}{c}{Learned} & \multicolumn{1}{c}{Observed} & \multicolumn{1}{c}{Learned} & \multicolumn{1}{c@{}}{Observed} \\
\midrule
Cube & 19.708 & 63.056 & \textbf{62.5} & 46.9 \\
PushT & 3.774 & 10.668 & 52.3 & \textbf{69.5} \\
Reacher & 74.724 & 154.351 & 54.7 & \textbf{75.8} \\
TwoRoom & 133.857 & 266.676 & \textbf{53.1} & 46.9 \\
\bottomrule
\end{tabular}
\endgroup

%% file: manuscript_fragments/ablation_table_wrapper.tex
\begin{table}[t]
\centering
\captionsetup{justification=raggedright,singlelinecheck=false}
\caption{\textbf{Target placement and retrieval timing determine how memory guides control.}
We average success (\%) across the four tasks. Predictions span $L=5$ actions, and $\ell$ sets the target's offset in the recorded trajectory.
When varying memory size, we keep the default target offset. In the fixed-span setting, we use $h=H$ throughout execution.
Bold marks the highest mean per column within each panel.}
\label{tab:v2-ablation}
\input{tables_v2/ablation.tex}
\end{table}

%% file: tables_v2/ablation.tex
\begingroup
\small
\renewcommand{\arraystretch}{1}
\setlength{\tabcolsep}{7pt}
\begin{tabular}{@{}lrrrr@{}}
\multicolumn{5}{@{}l}{\textbf{(a) Target offset and memory size}} \\
\toprule
 & \multicolumn{2}{c}{AP-rank} & \multicolumn{2}{c}{AP-CEM} \\
\cmidrule(lr){2-3}\cmidrule(l){4-5}
Variant & \multicolumn{1}{c}{Standard} & \multicolumn{1}{c}{Perturbed} & \multicolumn{1}{c}{Standard} & \multicolumn{1}{c@{}}{Perturbed} \\
\midrule
Target $\ell=1$ & 69.5 & 61.3 & 43.0 & 37.9 \\
Target $\ell=3$ & 78.5 & 68.5 & \textbf{69.9} & \textbf{62.6} \\
Target $\ell=5$ (default) & \textbf{83.8} & \textbf{72.3} & 59.8 & 55.1 \\
Target $\ell=10$ & 81.6 & 68.4 & 54.1 & 52.7 \\
\addlinespace[4pt]
Memory 10\% & 77.5 & 68.9 & 56.6 & 52.1 \\
Memory 25\% & 80.1 & 70.4 & 58.4 & 53.5 \\
Memory 50\% & 82.0 & 71.0 & 58.6 & 53.8 \\
\midrule
\multicolumn{5}{@{}l}{\textbf{(b) Retrieval design for AP-rank}} \\
Variant & \multicolumn{2}{c}{Standard} & \multicolumn{2}{c}{Perturbed} \\
\cmidrule(lr){2-3}\cmidrule(l){4-5}
Default & \multicolumn{2}{c}{83.8} & \multicolumn{2}{c}{72.3} \\
Key without distant endpoint & \multicolumn{2}{c}{\textbf{85.7}} & \multicolumn{2}{c}{\textbf{72.9}} \\
Key without displacement & \multicolumn{2}{c}{83.0} & \multicolumn{2}{c}{72.1} \\
Fixed retrieval span & \multicolumn{2}{c}{9.6} & \multicolumn{2}{c}{11.3} \\
\bottomrule
\end{tabular}
\endgroup

%% file: manuscript_fragments/diagnostic_table_wrapper.tex
\begin{table}[t]
\centering
\caption{\textbf{Live-state prediction improves endpoint accuracy and action selection.}
We report mean squared latent distances in task-specific units, with lower values indicating better performance. Bold marks each metric's minimum.
Displacement transfers recorded motion, whereas Direct executes the closest record's action.
We exclude starts if any candidate ends before five actions (at most 3 of 128 standard and 6 of 256 perturbed starts per task).}
\label{tab:confirmation-diagnostics}
\input{tables_v2/diagnostic.tex}
\end{table}

%% file: tables_v2/diagnostic.tex
\begingroup
\small
\setlength{\tabcolsep}{6pt}
\begin{tabular}{@{}llrrrrr@{}}
\toprule
 &  & \multicolumn{2}{c}{Endpoint error} & \multicolumn{3}{c}{Selection regret} \\
\cmidrule(lr){3-4}\cmidrule(l){5-7}
Task & Start & \multicolumn{1}{c}{Predictor} & \multicolumn{1}{c}{Displ.} & \multicolumn{1}{c}{Predictor} & \multicolumn{1}{c}{Displ.} & \multicolumn{1}{c@{}}{Direct} \\
\midrule
Cube & Standard & \textbf{7.538} & 35.344 & \textbf{5.332} & 7.535 & 11.375 \\
 & Perturbed & \textbf{37.726} & 66.860 & \textbf{8.472} & 11.731 & 14.903 \\
\addlinespace[4pt]
PushT & Standard & \textbf{1.532} & 3.078 & \textbf{0.200} & 0.580 & 0.849 \\
 & Perturbed & \textbf{2.980} & 9.562 & \textbf{0.753} & 2.418 & 4.449 \\
\addlinespace[4pt]
Reacher & Standard & \textbf{25.243} & 51.917 & \textbf{8.644} & 15.058 & 18.640 \\
 & Perturbed & \textbf{25.587} & 50.862 & \textbf{7.401} & 12.558 & 16.296 \\
\addlinespace[4pt]
TwoRoom & Standard & \textbf{13.560} & 162.035 & \textbf{1.433} & 29.838 & 41.749 \\
 & Perturbed & \textbf{12.532} & 160.497 & \textbf{1.283} & 20.154 & 37.307 \\
\bottomrule
\end{tabular}
\endgroup

%% file: sections/02_related.tex
\section{Related Work}
\label{sec:related}

\paragraph{Planning objectives.}
Artificial potential fields turn goals and obstacles into control objectives~\citep{khatib1986potential}. Local descent on such fields can be trapped at nonglobal minima~\citep{koren1991potential,barraquand1991rpp}, and with known free-space geometry, navigation functions can remove them~\citep{rimon1992exact}. Latent planners score predicted features~\citep{sobal2025pldm}, while Universal Planning Networks and Monotone Cost Ranking learn representations or costs for control~\citep{srinivas2018upn,chahe2026monotone}. \citet{terver2025jepawm} study optimizers and latent costs, showing that accurate prediction need not yield successful planning. Concurrent studies test decoded-position and learned temporal costs on TwoRoom~\citep{singh2026objective} or learn costs over latent rollouts~\citep{huang2026trajlewm}. We isolate target-induced stalling with exact dynamics and optimal short-horizon search, then compare targets with a fixed learned predictor, metric, and search.

\paragraph{Recorded and generated intermediate targets.}
SPTM and SoRB search graphs of observations using learned connectivity or distances~\citep{savinov2018sptm,eysenbach2019sorb}. PALMER likewise uses learned reachability to stitch recorded segments~\citep{beker2022palmer}. HWM predicts subgoals with a high-level world model~\citep{zhang2026hwm}, and Hi-LeWM improves subgoal quality by constraining high-level search near encoded training sequences~\citep{caselli2026hilewm}. SAGE learns target and action generators with frozen LeWM, using expert trajectories and low-dimensional state~\citep{cheng2026sage}. Its generator-only variant is closest to our observation-only CEM because it scores CEM candidates against a learned target. AP instead takes an observed successor from a retrieved segment as its target. Our learned-target baseline compares successor prediction with retrieval using the same visual inputs and memory.

\paragraph{Experience for action and target choice.}
VINN predicts actions by local regression on visual neighbors~\citep{pari2022vinn}. Our Direct baseline tests immediate action reuse under our current--goal retrieval rule. Recorded experience has also been used to estimate transitions~\citep{malato2025memory}, to aid hybrid LeWM planning~\citep{gao2026imwm}, to learn policies and trajectory models~\citep{florence2022ibc,chi2023diffusion,janner2022diffuser}, to relabel goals in hindsight~\citep{andrychowicz2017her}, and to learn rewards from video~\citep{chen2021dvd}. OGBench provides offline goal-conditioned benchmarks built from such data~\citep{park2025ogbench}. In AP, we use experience to choose the destination and live-state prediction to decide which synthesized or retrieved action approaches it.

%% file: sections/06_conclusion.tex
\section{Conclusion}
\label{sec:conclusion}

Where a planner aims can matter as much as what its model predicts. Our construction proves that goal-distance scoring can freeze a planner even when its predictions are perfect. On the LeWM tasks, a nearby target taken from recorded experience lets the same frozen models and planners reach goals they otherwise miss, with gains that additional final-goal search does not match. Success on goal-reaching benchmarks, a standard yardstick for world models, therefore measures the planning target as well as the predictor. Anchored Planning obtains these gains from three design choices: a nearby target, a retrieval span that shrinks during execution, and prediction from the live state. Because it needs nothing beyond an encoder, a predictor, and recorded experience, it applies directly to other world-model families that score predicted embeddings. Retargeting thus offers a direct way to draw more control from world models that already exist.

%% file: sections/A0_fresh_protocol.tex
\section{Evaluation Protocol}
\label{app:fresh-protocol}

\begin{table}[!htbp]
\centering
\caption{\textbf{Evaluation inputs.} $H$ is the recorded goal offset, $B$ the primitive-action allowance, and memory the number of eligible episodes.}
\label{tab:setup}
\input{tables/evaluation_inputs_main.tex}
\end{table}

\subsection{Query construction and retrieval memory}
Memory consists of the offline training trajectories provided with the task models. LeWM can be trained on both pseudo-expert and exploratory behavior-policy data~\citep{maes2026lewm}. Both memory settings provide ordered observations, temporal indices, and episode boundaries. Direct and the ranking controllers can also use the recorded action values.

The main study uses 128 reserved episodes per task, with one standard start and two preassigned five-action prefixes. Each prefix is a recorded five-action block from the task's training trajectories. Before control, we clip its actions to the action bounds and execute them while keeping the goal fixed. We sample queries and assign prefixes with fixed random seeds, so every controller sees the same queries, goals, and prefixes. Queries come from the source validation split for Cube, Reacher, and TwoRoom and from the source training split for PushT. Queries use previously unused eligible episodes. For every task, all 128 query episodes are excluded from the retrieval memory and from the trajectories that supply perturbation prefixes. The remaining memories contain 8,000, 16,688, 8,000, and 8,000 episodes for Cube, PushT, Reacher, and TwoRoom. We build the retrieval indices and horizon-specific statistics from these memories with the same frozen features and Equation~\eqref{eq:retrieval}.

All tasks reuse their supplied pretrained checkpoints. Retrieval memory, query episodes, and perturbation prefixes are fixed before evaluation and shared by all controllers.

Table~\ref{tab:setup} lists the task-specific goal offset $H$ and allowance $B$. The perturbation prefix is executed before the policy's action allowance begins. Policy time then starts at zero and increments by the number of executed primitives, so retrieval uses $\max(5,H-t)$. The environment clock continues through the prefix. If an episode succeeds or terminates before policy entry, we assign the same outcome to every controller. This includes success at the initial state or during a prefix. Success rates include these outcomes for all assigned queries.

\subsection{Physical success}
Evaluation uses the following task-specific physical predicates. They are distinct from the latent distances used for action selection.
\begin{itemize}
\item \textbf{Cube:} the selected object's three-dimensional position is at most $0.04$\,m from its target. The single-cube data-collection environment hides the target marker. Neither orientation nor gripper position is part of this predicate.
\item \textbf{PushT:} the Euclidean norm over the concatenated agent and object position errors is strictly below 20 environment units, and the shortest circular object-angle error is below $\pi/9$. 
\item \textbf{Reacher:} each of the two joint-position errors has absolute value below $0.05$\,rad. The errors use direct joint-angle differences.
\item \textbf{TwoRoom:} the Euclidean distance from the two-dimensional agent position to its goal is below 16 environment units.
\end{itemize}
We check success initially and after every primitive, using the task predicate at the stopping state to determine the outcome. A controller stops on success, on terminal failure, or when its primitive-action allowance runs out, so execution can end before a full block. The matched construction restores the same physical state and goal before running each controller. We use simulator coordinates to restore matched starts and to judge physical success, whereas the controller scores encoded images.

\subsection{Shared model and search implementation}
The pretrained model belongs to the joint-embedding family~\citep{assran2023ijepa,balestriero2025lejepa,maes2026lewm}. All controllers use the same task-specific frozen encoder. The predictive controllers also share the frozen forward model. Images are resized to 224 pixels with ImageNet normalization, and predicted endpoints are compared in the 192-dimensional projected latent space. Each prediction starts from the current image alone. 

We keep LeWM's pretrained action normalizer unchanged and compute retrieval statistics from the active memory. For a raw action block $u$, $F(z,u)$ applies this normalization before prediction. The recorded-action controllers and CEM clip proposals to the action bounds before scoring and execute exactly the clipped actions. Clipping leaves the target unchanged: it remains the encoding of the recorded successor image.

Unlike LeWM's multiblock planner, our three CEM variants search over a single five-action block. They use 30 iterations, 300 candidates, 30 elites, an initial normalized mean of zero, initial standard deviation $1/3$, and an elite-standard-deviation floor of $10^{-5}$. The search maintains means and standard deviations in normalized coordinates and includes the current mean as the first candidate in each population. For scoring, we clip candidates to the action bounds expressed in normalized coordinates. The elite mean and standard deviation are computed from the unclipped proposals. After selection, we map the chosen block back to raw actions and clip it to the action bounds to remove round-off at the boundary. Finally, we compare the clipped final mean with the clipped initial mean and keep the initial mean unless the final mean has lower cost. For a given task and query, all three target variants and all start conditions use the same random seed. Each search starts from the fixed Gaussian prior. After the 30 populations, the final and initial means are scored together, giving 31 batches and 9,002 predicted blocks per decision.

Direct and the three ranking controllers use the same retrieval rule, and its eight starts can come from the same episode. AP-rank and final-goal rank share the first candidate predictions exactly, and the CEM controllers share the first prior and search noise. Once the policies choose different actions, their later states, retrieved records, and random draws can differ.

\subsection{Aggregation and prediction work}
\label{app:fresh-results}
To aggregate perturbed success rates, we first average the two starts of each query and then average queries within a task. All four tasks receive equal weight in task means. Throughout this appendix, we report standard starts and the perturbation average. Original-offset results are reused across comparisons.

To measure prediction work, we count the five-action blocks predicted until the episode ends. Ranking evaluates eight blocks in one batch per decision, and single-block CEM evaluates $300I+2$ blocks for $I$ iterations. LeWM evaluates 45,000 blocks per decision using its multiblock search and executes up to 25 primitive actions before replanning. Direct uses no transition prediction for action selection. Encoding, target construction, retrieval, and interaction also contribute to elapsed runtime.

\subsection{Observation-only deployment}
\label{app:observation-memory}

Observation-only memory supplies ordered images, their frozen features, episode boundaries, and temporal indices. These identify the retrieved segment and the observed target. CEM generates candidate actions and passes them to the frozen predictor, which was pretrained with actions. Both memory settings use the same retrieval and target construction.

%% file: tables/evaluation_inputs_main.tex
\begin{tabular}{@{}lrrrr@{}}
\toprule
Task & \multicolumn{1}{c}{Cube} & \multicolumn{1}{c}{PushT} & \multicolumn{1}{c}{Reacher} & \multicolumn{1}{c@{}}{TwoRoom} \\
\midrule
$H$ & 150 & 140 & 150 & 100 \\
$B$ & 150 & 280 & 300 & 200 \\
Memory & 8,000 & 16,688 & 8,000 & 8,000 \\
\bottomrule
\end{tabular}

%% file: sections/A4_reproducibility.tex
\subsection{Algorithm and implementation details}
\label{app:reproducibility}

Algorithm~\ref{alg:anchored} gives the AP-rank procedure used in every main-study task and condition. Section~\ref{sec:anchoring} defines how it retrieves a record, sets a target, selects a clipped action block, and replans after execution. The preceding protocol defines the paired starts and physical success criteria.

\input{tables/ap_rank_algorithm.tex}

The LeWM planner keeps its action mapping and execution settings, which Appendix~\ref{app:lewm-audit} describes, and the independent supporting study uses the mapping described in Appendix~\ref{app:independent}.

The current observation is encoded with the same frozen encoder as the memory before retrieval and scoring. For each horizon $h$, the retrieval-key statistics are computed from the records in $\mathcal B_h$. Coordinate-wise population standard deviations are floored at $10^{-4}$, and the source, distant endpoint, and displacement parts have equal weights. $\mathcal B_h$ and these statistics stay fixed during evaluation. Unlike our controllers, LeWM's planner accumulates up to three context blocks within its rollout.

All experiments run on NVIDIA H20 GPUs with PyTorch. The tasks use their MuJoCo-based simulators.

%% file: tables/ap_rank_algorithm.tex
\begin{algorithm}[t]
\caption{AP-rank with an observed target for a frozen predictor}
\label{alg:anchored}
\begin{algorithmic}[1]
\Require frozen $E,F$; memory $\{\mathcal B_h\}$; current image $o_0$; goal image $o_g$; supplied $H$; action allowance $B$
\State $L\gets5$; $k\gets8$; $z_g\gets E(o_g)$; $t\gets0$
\While{$t<B$, the goal is not reached, and the environment is active}
  \State $z_t\gets E(o_t)$; $h\gets\max\{L,H-t\}$
  \State retrieve $(s_1,\ldots,s_k)$ from $\mathcal B_h$ using Equation~\eqref{eq:retrieval}
  \State $q_{\rm obs}\gets E(o_{s_1+L})$; $\mathcal C\gets\{P(u_{s_1}),\ldots,P(u_{s_k})\}$
  \State predict $\{F(z_t,u):u\in\mathcal C\}$ in one batch
  \State select $\hat u$ by Equation~\eqref{eq:anchored}, breaking ties by retrieval order
  \State execute up to $L$ primitives from $\hat u$, stopping at success, termination, or $B$
  \State increase $t$ by the executed primitives and observe the next image $o_t$
\EndWhile
\end{algorithmic}
\end{algorithm}

%% file: sections/A1_theory.tex
\section{Properties of Target-Based Selection}
\label{app:theory}

\subsection{A continuous counterexample with smooth dynamics}
\label{app:smooth-counterexample}

\begin{proof}[Proof of Proposition~\ref{prop:exact-target}]
Let $\mathcal X=\{c(s):s\in\mathbb R\}$, where $c(s)=(s,s^2)$, and use the identity encoder. A primitive action $a\in[0,1/8]$ has transition
\[
T((x,y),a)=(x+a,\ y+2xa+a^2).
\]
This polynomial map is the unit-time flow of $\dot x=a$, $\dot y=2xa$ under constant control and preserves $\mathcal X$. Thus a five-action block $u\in\mathcal C=[0,1/8]^5$ has exact endpoint
\[
F(c(s),u)=c\!\left(s+\sum_{i=1}^5a_i\right).
\]
Both controllers globally minimize squared endpoint distance over the entire continuous set $\mathcal C$, execute the chosen block, and observe again. Only the target differs.

For the final goal $g=c(2)$, write
\[
D(s)=\|c(s)-g\|^2=(s-2)^2+(s^2-4)^2,
\qquad D'(s)=2(s-2)(2s^2+4s+1).
\]
The quadratic roots are $-1\pm1/\sqrt2$, so $D'(s)>0$ throughout $[-3/2,-3/8]$. From any $s_0\in[-3/2,-1]$, an admissible block ends at $c(s_0+d_u)$ with $0\leq d_u\leq5/8$, within that increasing region. Every $d_u>0$ increases final-goal distance. Nonnegative actions make the zero block the unique global minimizer. It leaves the state unchanged, so final-goal selection never reaches $g$.

Now take one feasible recorded trajectory with consecutive block endpoints
\[
(r_0,r_1,r_2,r_3,r_4,r_5,r_6)
=\left(-\frac32,-\frac78,-\frac14,\frac38,1,\frac{13}{8},2\right).
\]
Its first five blocks use $a_i=1/8$ and the last uses $a_i=3/40$. Supply the observed targets $c(r_1),\ldots,c(r_6)$ in order. For every allowed $s_0$, the first increment $r_1-s_0\in[1/8,5/8]$ is feasible. Each later increment is also positive and at most $5/8$. At each decision a block attains zero target cost, so every global minimizer reaches that target exactly. Induction gives the final goal in six blocks. Ties between primitive sequences do not change their endpoint, and primitive-level stopping can only shorten the final block.
\end{proof}

The dynamics admit a continuous route from every stated start. Failure under final-goal scoring comes from rejecting the temporary increase in distance needed to follow it. The prescribed recorded sequence isolates the target's role: with the same exact dynamics and the same one-block global minimization, it yields a goal-reaching controller.

\paragraph{The failure persists with symmetric actions.}
Allow $a_i\in[-1/8,1/8]$ and write $\alpha=-1-1/\sqrt2$, $\beta=-1+1/\sqrt2$. Below $\beta$, $D$ decreases to $\alpha$ and then increases. From $s_0\in[-3/2,-1]$, the reachable interval $[s_0-5/8,s_0+5/8]$ lies below $\beta$. Its minimizing coordinate is $\max\{\alpha,s_0-5/8\}$, and repeating this rule reaches $\alpha$ within two blocks. Since $\alpha+5/8<\beta$, subsequent minimizing blocks also end at $\alpha$. The same observed-target sequence remains feasible from each original start and reaches the goal. Allowing both action directions thus preserves failure under final-goal scoring.

\subsection{How endpoint errors affect action selection}
\label{app:score-sensitivity}
At a fixed live state and target $q$, let $e_u$ be an action block's exact encoded endpoint and $r_u=F(z_t,u)-e_u$ its prediction error. The resulting error in target cost is
\[
\xi_u(q)=J_q(u)-\|e_u-q\|^2
=2(e_u-q)^\top r_u+\|r_u\|^2.
\]
The predicted cost difference between $u$ and $v$ therefore differs from the exact one by $\xi_u(q)-\xi_v(q)$. Two endpoint errors with the same size can change action costs differently: the dot product depends on the direction of each error relative to the target. If both costs change by the same amount, their ordering stays unchanged.

If exact target costs favor $u$ over $v$ by a gap $m>0$, that ordering is preserved whenever $\xi_u(q)-\xi_v(q)<m$. This explains why endpoint error and selection regret provide complementary information in Table~\ref{tab:confirmation-diagnostics}: the first measures endpoint accuracy, while the second measures the cost of the action chosen under the target.

Proposition~\ref{prop:exact-target} concerns the target itself. With $r_u=0$, every predicted target cost is exact, yet final-goal scoring still favors staying still in the construction. Changing the target changes those costs without changing the dynamics.

%% file: sections/C_rq1.tex
\section{RQ1: Target Choice and Control}
\label{app:v2-main}
The main target comparisons use the shared protocol in Appendix~\ref{app:fresh-protocol}. This section gives target-model selection and the physical behavior behind the success rates in Table~\ref{tab:main-results}.

\subsection{Learned intermediate targets}
\label{app:v2-learned-target}

We train a target predictor on the same memory episodes available to Anchored
Planning. For each task, a fixed episode split reserves five percent of memory
for validation. Control-query episodes are excluded from both splits. To form a
training example, we sample an offset $h$ uniformly from $5,\ldots,H_{\rm task}$,
where $H_{\rm task}$ is the task's main goal offset, and then choose a valid source
position $s$ uniformly among the training positions at that offset.
The input is $(z_s,z_{s+h},z_{s+h}-z_s,h/H_{\rm task})$, and the target is the observed
successor $z_{s+5}$. All encodings come from the frozen LeWM encoder.

The predictor is an MLP with three hidden layers, each using LayerNorm and GELU.
We compare widths 512 and 1024, with either an absolute output or a residual
added to $z_s$. Inputs and outputs are standardized coordinate-wise using
262,144 sampled training transitions. We train with AdamW using batch size 1024, learning rate
$3\times10^{-4}$, weight decay $10^{-4}$, and cosine decay over at most 50,000
updates. Validation uses the same 4096 held-out transitions for all four
configurations. We evaluate every 500 updates and stop after ten checks without
improvement, once at least 5000 updates have run. The selected checkpoint
minimizes unstandardized latent-coordinate MSE, which gives absolute and
residual models the same selection scale.

At execution, the MLP receives the current encoding, final-goal encoding,
their difference, and $h(t)/H_{\rm task}$. Its output replaces the scoring target for either CEM or ranking. CEM keeps its initialization and search settings, ranking scores the same retrieved blocks, and both keep the frozen predictor, action clipping, and execution rule of the main study.

\begin{table}[!htbp]
\centering
\caption{\textbf{Selecting learned targets on held-out memory.}
We average squared error over latent coordinates to compute validation MSE.
Bold marks the selected configuration.}
\label{tab:v2-target-models}
\input{tables_v2/target_models.tex}
\end{table}

Table~\ref{tab:target-quality-control} relates target accuracy to control. To measure target accuracy, we use 4096 held-out transitions per task and exclude each source episode from retrieval. Successor errors and nearest-memory distances are squared latent distances. An observed target has zero distance to memory by construction.

\subsection{Physical behavior along a rollout}
\label{app:v2-physical-behavior}
Physical error is measured in units of the success thresholds. For Cube and TwoRoom, $\hat e_t$ is position error divided by $0.04$\,m and 16 units, respectively. For Reacher, it is the larger joint error divided by $0.05$\,rad. For PushT, it is the maximum of combined position error divided by 20 and the wrapped object-angle error divided by $\pi/9$. We apply the success criteria defined in Appendix~\ref{app:fresh-protocol}.

A failed rollout stalls when its task-relevant coordinates remain within one quarter of their success thresholds from the state at the beginning of the last $W=10$ decisions. Shorter rollouts use their available window. A successful rollout makes a detour when $\hat e_t$ rises by at least $\delta=0.5$ above its previous running minimum.

Table~\ref{tab:v2-mechanism} gives the rates with their eligible counts, and Table~\ref{tab:behavior-sensitivity} varies the window and threshold for the three CEM target rules.

\input{tables_compact/behavior.tex}
\input{tables_compact/sensitivity.tex}

\subsection{A matched PushT rollout}
\label{app:qualitative-rollout}
We examine main-study queries in a fixed order, testing the first prefix before the second for each query. Figure~\ref{fig:behavior} shows the physical states throughout the first rollout in which AP-rank succeeds and Direct fails.

\begin{figure}[!htbp]
\centering
\includegraphics[width=3.5in]{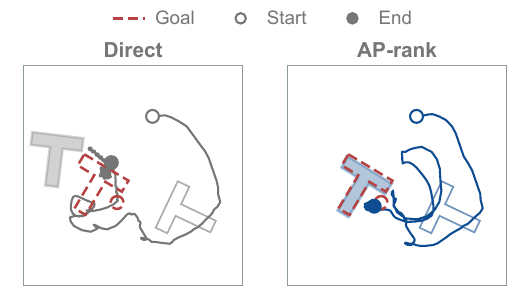}
\caption{\textbf{Predictive ranking changes a matched PushT execution.} Both controllers share the perturbed start and goal. Curves trace the agent's path. Circles show the agent, and T shapes show the pushed block. Hollow and filled shapes show initial and terminal poses, and dashed outlines mark the goal.}
\label{fig:behavior}
\end{figure}

%% file: tables_v2/target_models.tex
\begingroup
\small
\setlength{\tabcolsep}{4pt}
\begin{tabular*}{\linewidth}{@{\extracolsep{\fill}}lrllr@{}}
\toprule
Task & \multicolumn{1}{c}{Width} & Output & Best step & \multicolumn{1}{c@{}}{Validation MSE} \\
\midrule
Cube & 512 & Absolute & 48,500 & 0.104926 \\
 & 512 & Residual & 48,500 & 0.104579 \\
 & 1024 & Absolute & 41,000 & 0.104422 \\
 & 1024 & Residual & 48,500 & \textbf{0.102646} \\
\addlinespace[4pt]
PushT & 512 & Absolute & 48,000 & 0.023950 \\
 & 512 & Residual & 48,000 & 0.024007 \\
 & 1024 & Absolute & 47,500 & 0.020210 \\
 & 1024 & Residual & 45,500 & \textbf{0.019658} \\
\addlinespace[4pt]
Reacher & 512 & Absolute & 11,500 & 0.391840 \\
 & 512 & Residual & 21,500 & \textbf{0.389188} \\
 & 1024 & Absolute & 11,500 & 0.391255 \\
 & 1024 & Residual & 21,500 & 0.389396 \\
\addlinespace[4pt]
TwoRoom & 512 & Absolute & 14,500 & \textbf{0.697172} \\
 & 512 & Residual & 15,500 & 0.702978 \\
 & 1024 & Absolute & 8,000 & 0.697239 \\
 & 1024 & Residual & 7,500 & 0.704444 \\
\bottomrule
\end{tabular*}
\endgroup

%% file: tables_compact/behavior.tex
\begin{table}[!htbp]
\centering
\caption{\textbf{Physical behavior at the original goal offsets.} Each percentage is followed by the number of positive cases and eligible rollouts. With $W=10$, we measure stalling over failed rollouts in which the policy made at least one decision. With $\delta=0.5$, we measure detours over successful rollouts that entered policy control. For perturbed behavior, we pool eligible rollouts from both prefixes.}
\label{tab:v2-mechanism}
\begingroup\small
\setlength{\tabcolsep}{3pt}
\begin{tabular}{@{}llrrrr@{}}
\toprule
Task & Controller & \multicolumn{2}{c}{Standard} & \multicolumn{2}{c}{Perturbed} \\
\cmidrule(lr){3-4}\cmidrule(lr){5-6}
 &  & \multicolumn{1}{c}{Stalling} & \multicolumn{1}{c}{Detour} & \multicolumn{1}{c}{Stalling} & \multicolumn{1}{c@{}}{Detour} \\
\midrule
Cube & LeWM & 39.8 (43/108) & 10.5 (2/19) & 48.2 (108/224) & 6.7 (2/30) \\
 & CEM (final goal) & 100.0 (119/119) & 37.5 (3/8) & 100.0 (249/249) & 40.0 (2/5) \\
 & CEM (learned) & 79.2 (38/48) & 60.8 (48/79) & 89.2 (140/157) & 68.0 (66/97) \\
 & AP-CEM & 86.8 (59/68) & 66.1 (39/59) & 92.2 (166/180) & 73.0 (54/74) \\
 & CEM (transported) & 91.3 (73/80) & 68.1 (32/47) & 95.4 (186/195) & 71.2 (42/59) \\
 & Direct & 87.0 (20/23) & 60.6 (63/104) & 94.0 (126/134) & 64.2 (77/120) \\
 & Final-goal rank & 93.8 (45/48) & 55.7 (44/79) & 99.3 (148/149) & 61.0 (64/105) \\
 & Learned-target rank & 100.0 (11/11) & 64.7 (75/116) & 89.1 (114/128) & 63.5 (80/126) \\
 & AP-rank & 100.0 (34/34) & 61.3 (57/93) & 92.8 (129/139) & 67.0 (77/115) \\
\addlinespace[3pt]
PushT & LeWM & 28.6 (34/119) & 88.9 (8/9) & 28.0 (65/232) & 83.3 (20/24) \\
 & CEM (final goal) & 80.0 (100/125) & 33.3 (1/3) & 78.7 (196/249) & 42.9 (3/7) \\
 & CEM (learned) & 80.3 (49/61) & 98.5 (66/67) & 81.7 (107/131) & 99.2 (124/125) \\
 & AP-CEM & 92.3 (36/39) & 100.0 (89/89) & 83.3 (75/90) & 100.0 (166/166) \\
 & CEM (transported) & 87.8 (43/49) & 100.0 (79/79) & 78.2 (111/142) & 100.0 (114/114) \\
 & Direct & 95.2 (59/62) & 100.0 (66/66) & 91.0 (191/210) & 100.0 (46/46) \\
 & Final-goal rank & 90.5 (86/95) & 100.0 (33/33) & 88.3 (197/223) & 100.0 (33/33) \\
 & Learned-target rank & 93.1 (54/58) & 100.0 (70/70) & 87.3 (144/165) & 98.9 (90/91) \\
 & AP-rank & 90.7 (39/43) & 100.0 (85/85) & 89.6 (120/134) & 100.0 (122/122) \\
\addlinespace[3pt]
Reacher & LeWM & 0.0 (0/43) & 87.1 (74/85) & 0.0 (0/80) & 90.8 (157/173) \\
 & CEM (final goal) & 0.0 (0/95) & 51.5 (17/33) & 0.0 (0/198) & 47.3 (26/55) \\
 & CEM (learned) & 0.0 (0/58) & 92.9 (65/70) & 0.0 (0/105) & 93.9 (139/148) \\
 & AP-CEM & 0.0 (0/31) & 96.9 (94/97) & 0.0 (0/54) & 97.5 (194/199) \\
 & CEM (transported) & 0.0 (0/24) & 96.2 (100/104) & 0.0 (0/50) & 97.5 (198/203) \\
 & Direct & 0.0 (0/1) & 97.6 (124/127) & 0.0 (0/8) & 99.2 (243/245) \\
 & Final-goal rank & 0.0 (0/12) & 94.0 (109/116) & 0.0 (0/32) & 95.0 (210/221) \\
 & Learned-target rank & 0.0 (0/6) & 100.0 (122/122) & 0.0 (0/8) & 99.6 (244/245) \\
 & AP-rank & 0.0 (0/4) & 98.4 (122/124) & 0.0 (0/11) & 99.6 (241/242) \\
\addlinespace[3pt]
TwoRoom & LeWM & 0.0 (0/106) & 68.2 (15/22) & 0.0 (0/213) & 51.2 (22/43) \\
 & CEM (final goal) & 33.3 (42/126) & 0.0 (0/2) & 32.1 (81/252) & 0.0 (0/4) \\
 & CEM (learned) & 0.0 (0/60) & 27.9 (19/68) & 0.0 (0/115) & 27.0 (38/141) \\
 & AP-CEM & 19.1 (13/68) & 73.3 (44/60) & 15.4 (21/136) & 75.8 (91/120) \\
 & CEM (transported) & 41.7 (35/84) & 59.1 (26/44) & 39.8 (66/166) & 64.4 (58/90) \\
 & Direct & 0.0 (0/1) & 96.9 (123/127) & -- (0/0) & 99.2 (254/256) \\
 & Final-goal rank & 7.1 (1/14) & 93.9 (107/114) & 6.9 (2/29) & 96.5 (219/227) \\
 & Learned-target rank & 25.0 (1/4) & 97.6 (121/124) & 0.0 (0/4) & 96.8 (244/252) \\
 & AP-rank & 0.0 (0/2) & 98.4 (124/126) & -- (0/0) & 99.2 (254/256) \\
\bottomrule
\end{tabular}
\endgroup
\end{table}

%% file: tables_compact/sensitivity.tex
\begin{table}[!htbp]
\centering
\caption{\textbf{Behavior across thresholds.} We average rates (\%) equally across the four tasks, computing each rate from that task's eligible rollout counts.}
\label{tab:behavior-sensitivity}
\begingroup\small
\setlength{\tabcolsep}{4pt}
\begin{tabular}{@{}llrrrrrr@{}}
\toprule
Controller & Start & \multicolumn{3}{c}{Stalling window} & \multicolumn{3}{c}{Detour threshold} \\
\cmidrule(lr){3-5}\cmidrule(lr){6-8}
 &  & \multicolumn{1}{c}{$W=5$} & \multicolumn{1}{c}{$10$} & \multicolumn{1}{c}{$20$} & \multicolumn{1}{c}{$\delta=0.25$} & \multicolumn{1}{c}{$0.5$} & \multicolumn{1}{c@{}}{$1.0$} \\
\midrule
CEM (final goal) & Standard & 55.1 & 53.3 & 49.9 & 68.6 & 30.6 & 29.1 \\
CEM (final goal) & Perturbed & 54.9 & 52.7 & 49.7 & 46.4 & 32.5 & 19.0 \\
CEM (learned) & Standard & 46.1 & 39.9 & 31.4 & 77.6 & 70.0 & 53.7 \\
CEM (learned) & Perturbed & 45.6 & 42.7 & 32.1 & 79.3 & 72.0 & 58.9 \\
AP-CEM & Standard & 51.7 & 49.5 & 40.5 & 90.6 & 84.1 & 70.7 \\
AP-CEM & Perturbed & 52.2 & 47.7 & 39.3 & 91.5 & 86.6 & 71.8 \\
\bottomrule
\end{tabular}
\endgroup
\end{table}

%% file: sections/D_rq2.tex
\section{RQ2: Search Budget and Goal Distance}
\label{app:rq2}
\label{app:v2-interventions}
\subsection{Search budgets}
In the budget sweep, we vary CEM iterations over $\{1,2,5,10,30\}$ while preserving the main queries, perturbation prefixes, frozen models, and action clipping. At every budget, we use the same query seed and initial proposals, then draw candidates for the specified number of iterations. The largest-budget points reuse the main outcomes. Table~\ref{tab:budget-full} gives both start conditions.

\input{tables_compact/budget.tex}
\FloatBarrier

\subsection{LeWM configuration}
\label{app:lewm-audit}
The LeWM planner is implemented in the \texttt{stable-worldmodel} library (version 0.0.6). Its five model blocks each contain five primitive actions. Before replanning, the policy executes all five blocks unless success, termination, or the action allowance stops it earlier. CEM uses 300 candidates, 30 iterations, 30 elites, initial standard deviation $1$, and the LeWM history and mean updates. Its actions are de-normalized with the pretrained normalizer and sent directly to the environment. Starts, goals, action allowances, checkpoints, and success predicates match the target comparisons. Our CEM variants, by contrast, share one search setting and differ only in the target.

\paragraph{LeWM evaluation protocol.}
The LeWM evaluator samples valid start states from the supplied datasets, takes the goal image 24 actions after the start, and allows 50 actions. Our $H=25$ queries place the goal 25 actions after the start and allow 25 actions on Cube and 50 on the other tasks. The two protocols also start episodes differently. The LeWM evaluator scores the dataset image of the start and, on Cube, restores the full recorded simulator state, including joint positions and velocities. We instead render the restored start, and every controller in our study begins from a state reconstructed from the public observation. Finally, the LeWM configuration uses 30 CEM iterations on every task, although \citet{maes2026lewm} report ten outside PushT.

Table~\ref{tab:lewm-released-check} lists the published rates next to runs of the LeWM evaluator and of our implementation with the same queries, action mapping, and CEM random seeds. The two implementations give the same outcome on 599 of 600 episodes.

\input{tables_compact/lewm_release_check.tex}

\Needspace{4\baselineskip}
On Cube and PushT, query sampling produces the largest change among the factors we varied one at a time: with the LeWM sampling rule, LeWM success in our $H=25$ setting rises from 32.0\% to 71.9\% on Cube and from 60.9\% to 82.8\% on PushT, whereas a 50-action allowance alone raises Cube to 36.7\%. Applying the LeWM success criteria to 512 main-study episodes changes no outcome.

\paragraph{Five-action replanning.}
The additional variant replans after at most five primitive actions instead of 25. It keeps the LeWM lookahead, candidate count, iterations, elites, initial deviation, and warm start, and uses the main queries, goals, and action allowances. Table~\ref{tab:lewm-replan5} reports its results.

\input{tables_compact/lewm_replan5.tex}
\FloatBarrier

\subsection{Recorded goal offsets}
Each goal is the observation at the selected offset from the original query source. We add $H=25$, $50$, and $100$ for Cube, PushT, and Reacher, and $H=25$ and $50$ for TwoRoom. We set the action allowance to preserve each task's original $B/H$. The learned target retains its training normalization $H_{\rm task}$. Original-offset points reuse the main outcomes. Table~\ref{tab:horizon-full} gives synthesis and ranking results.

\input{tables_compact/horizon.tex}

\Needspace{9\baselineskip}
\subsection{Prediction work}
Table~\ref{tab:work-summary} reports prediction work at the original offsets, counted as in Appendix~\ref{app:fresh-results}. LeWM can execute up to 25 actions between decisions, whereas the other predictive controllers execute up to five. 

\input{tables_compact/work.tex}

%% file: tables_compact/budget.tex
\begin{table}[!htbp]
\centering
\caption{\textbf{Success across search budgets.} We report success (\%) on the main queries at the original offsets. Only the number of CEM iterations changes, and each decision evaluates $300I+2$ predicted blocks. For perturbed starts, we average the two starts of each query.}
\label{tab:budget-full}
\begingroup\small
\setlength{\tabcolsep}{4pt}
\begin{tabular}{@{}lrrrrrrr@{}}
\toprule
Task & \multicolumn{1}{c}{Iterations} & \multicolumn{2}{c}{Final goal} & \multicolumn{2}{c}{Learned target} & \multicolumn{2}{c}{Observed target} \\
\cmidrule(lr){3-4}\cmidrule(lr){5-6}\cmidrule(lr){7-8}
 &  & \multicolumn{1}{c}{Standard} & \multicolumn{1}{c}{Perturbed} & \multicolumn{1}{c}{Standard} & \multicolumn{1}{c}{Perturbed} & \multicolumn{1}{c}{Standard} & \multicolumn{1}{c@{}}{Perturbed} \\
\midrule
Cube & 1 & 3.9 & 2.0 & 11.7 & 6.6 & 10.9 & 7.0 \\
 & 2 & 3.1 & 2.3 & 40.6 & 26.2 & 26.6 & 15.6 \\
 & 5 & 3.1 & 2.3 & 52.3 & 34.0 & 36.7 & 26.6 \\
 & 10 & 4.7 & 3.9 & 59.4 & 35.2 & 44.5 & 29.7 \\
 & 30 & 7.0 & 2.7 & 62.5 & 38.7 & 46.9 & 29.7 \\
\addlinespace[4pt]
PushT & 1 & 1.6 & 2.0 & 3.9 & 3.9 & 3.1 & 1.6 \\
 & 2 & 1.6 & 2.0 & 25.8 & 27.0 & 8.6 & 10.9 \\
 & 5 & 1.6 & 2.0 & 44.5 & 41.8 & 43.8 & 43.4 \\
 & 10 & 1.6 & 2.7 & 48.4 & 44.5 & 69.5 & 68.0 \\
 & 30 & 2.3 & 2.7 & 52.3 & 48.8 & 69.5 & 64.8 \\
\addlinespace[4pt]
Reacher & 1 & 16.4 & 15.2 & 48.4 & 50.0 & 59.4 & 56.6 \\
 & 2 & 18.8 & 18.4 & 50.8 & 53.9 & 64.8 & 67.2 \\
 & 5 & 24.2 & 20.3 & 60.2 & 59.8 & 78.9 & 75.8 \\
 & 10 & 23.4 & 23.4 & 63.3 & 60.5 & 85.9 & 75.4 \\
 & 30 & 25.8 & 22.7 & 54.7 & 59.0 & 75.8 & 78.9 \\
\addlinespace[4pt]
TwoRoom & 1 & 0.0 & 0.0 & 34.4 & 35.2 & 7.0 & 9.0 \\
 & 2 & 0.0 & 0.0 & 39.1 & 41.0 & 19.5 & 20.3 \\
 & 5 & 0.8 & 0.0 & 46.9 & 49.2 & 37.5 & 37.1 \\
 & 10 & 0.8 & 1.6 & 55.5 & 53.5 & 43.0 & 45.3 \\
 & 30 & 1.6 & 1.6 & 53.1 & 55.1 & 46.9 & 46.9 \\
\bottomrule
\end{tabular}
\endgroup
\end{table}

%% file: tables_compact/lewm_release_check.tex
\begin{table}[htbp]
\centering\small
\caption{\textbf{LeWM under its evaluation protocol.} We compare published success (\%) from Figure~6 of \citet{maes2026lewm} with measurements from the LeWM evaluator and our implementation on the same queries. Measurements are averaged over three seeds with 50 queries each, including the seed specified in the LeWM configuration.}
\label{tab:lewm-released-check}
\setlength{\tabcolsep}{7pt}
\begin{tabular}{lrrr}\toprule
Task & Paper report & LeWM evaluator & Our implementation\\
\midrule
Cube & 74.0 & 74.7 & 74.7 \\
PushT & 96.0 & 92.0 & 91.3 \\
Reacher & 86.0 & 80.0 & 80.0 \\
TwoRoom & 87.0 & 87.3 & 87.3 \\
\bottomrule\end{tabular}
\end{table}

%% file: tables_compact/lewm_replan5.tex
\begin{table}[htbp]
\centering\small
\caption{\textbf{LeWM with five-action replanning.} The variant keeps the 25-action lookahead and the LeWM CEM settings and changes only execution, which now covers one five-action block. Evaluation uses the main queries and original goal offsets, with 128 standard starts and 256 perturbed starts per task. Work is the mean number of predicted five-action blocks per episode, in thousands as in Table~\ref{tab:work-summary}.}
\label{tab:lewm-replan5}
\setlength{\tabcolsep}{7pt}
\begin{tabular}{lrrrr}\toprule
& \multicolumn{2}{c}{Success (\%)} & \multicolumn{2}{c}{Work ($10^3$ blocks)}\\
\cmidrule(lr){2-3}\cmidrule(lr){4-5}
Task & Standard & Perturbed & Standard & Perturbed\\\midrule
Cube & 13.3 & 13.3 & 1,189 & 1,197 \\
PushT & 7.0 & 8.2 & 2,402 & 2,374 \\
Reacher & 73.4 & 77.3 & 1,388 & 1,322 \\
TwoRoom & 27.3 & 25.8 & 1,500 & 1,509 \\
\bottomrule\end{tabular}
\end{table}

%% file: tables_compact/horizon.tex
\begin{table}[!t]
\centering
\caption{\textbf{Success across recorded goal offsets.} To compare success (\%) across offsets, we keep query sources fixed and move the goal to the observation $H$ actions later, preserving each task's $B/H$. Prediction lookahead and target-model training remain unchanged. Results are shown for standard starts (Std.) and for the average of each query's perturbed starts (Pert.).}
\label{tab:horizon-full}
\begingroup\small
\setlength{\tabcolsep}{3.5pt}
\textit{Action synthesis}\par
\smallskip
\begin{tabular}{@{}lrrrrrrrrr@{}}
\toprule
Task & \multicolumn{1}{c}{$H$} & \multicolumn{2}{c}{LeWM} & \multicolumn{2}{c}{Final goal} & \multicolumn{2}{c}{Learned} & \multicolumn{2}{c}{AP-CEM} \\
\cmidrule(lr){3-4}\cmidrule(lr){5-6}\cmidrule(lr){7-8}\cmidrule(lr){9-10}
 &  & \multicolumn{1}{c}{Std.} & \multicolumn{1}{c}{Pert.} & \multicolumn{1}{c}{Std.} & \multicolumn{1}{c}{Pert.} & \multicolumn{1}{c}{Std.} & \multicolumn{1}{c}{Pert.} & \multicolumn{1}{c}{Std.} & \multicolumn{1}{c@{}}{Pert.} \\
\midrule
Cube & 25 & 32.0 & 33.6 & 28.9 & 26.2 & 75.0 & 53.9 & 63.3 & 49.6 \\
 & 50 & 16.4 & 12.1 & 7.0 & 4.3 & 58.6 & 37.9 & 47.7 & 30.5 \\
 & 100 & 18.8 & 13.3 & 8.6 & 9.0 & 55.5 & 34.8 & 32.0 & 25.4 \\
 & 150 & 15.6 & 12.5 & 7.0 & 2.7 & 62.5 & 38.7 & 46.9 & 29.7 \\
\addlinespace[3pt]
PushT & 25 & 60.9 & 57.8 & 51.6 & 46.9 & 91.4 & 75.4 & 91.4 & 87.5 \\
 & 50 & 22.7 & 21.9 & 6.3 & 9.0 & 76.6 & 63.7 & 82.0 & 81.6 \\
 & 100 & 10.9 & 9.4 & 3.9 & 4.3 & 57.0 & 48.4 & 77.3 & 73.8 \\
 & 140 & 7.0 & 9.4 & 2.3 & 2.7 & 52.3 & 48.8 & 69.5 & 64.8 \\
\addlinespace[3pt]
Reacher & 25 & 68.8 & 67.6 & 66.4 & 66.0 & 57.8 & 57.8 & 67.2 & 71.1 \\
 & 50 & 95.3 & 92.6 & 57.0 & 52.3 & 53.1 & 53.5 & 81.3 & 76.2 \\
 & 100 & 83.6 & 82.0 & 35.2 & 35.2 & 48.4 & 52.3 & 84.4 & 76.6 \\
 & 150 & 66.4 & 68.8 & 25.8 & 22.7 & 54.7 & 59.0 & 75.8 & 78.9 \\
\addlinespace[3pt]
TwoRoom & 25 & 77.3 & 74.2 & 30.5 & 34.8 & 99.2 & 99.2 & 89.1 & 85.9 \\
 & 50 & 36.7 & 35.5 & 5.5 & 6.6 & 94.5 & 94.1 & 72.7 & 73.8 \\
 & 100 & 17.2 & 16.8 & 1.6 & 1.6 & 53.1 & 55.1 & 46.9 & 46.9 \\
\bottomrule
\end{tabular}
\par\medskip\textit{Recorded-action selection}\par
\smallskip
\begin{tabular}{@{}lrrrrrrr@{}}
\toprule
Task & \multicolumn{1}{c}{$H$} & \multicolumn{2}{c}{Final goal} & \multicolumn{2}{c}{Learned} & \multicolumn{2}{c}{AP-rank} \\
\cmidrule(lr){3-4}\cmidrule(lr){5-6}\cmidrule(lr){7-8}
 &  & \multicolumn{1}{c}{Std.} & \multicolumn{1}{c}{Pert.} & \multicolumn{1}{c}{Std.} & \multicolumn{1}{c}{Pert.} & \multicolumn{1}{c}{Std.} & \multicolumn{1}{c@{}}{Pert.} \\
\midrule
Cube & 25 & 92.2 & 60.9 & 88.3 & 57.8 & 92.2 & 59.0 \\
 & 50 & 68.0 & 43.8 & 83.6 & 48.4 & 70.3 & 44.1 \\
 & 100 & 59.4 & 34.8 & 78.9 & 43.8 & 64.8 & 37.1 \\
 & 150 & 62.5 & 41.8 & 91.4 & 50.0 & 73.4 & 45.7 \\
\addlinespace[3pt]
PushT & 25 & 89.8 & 68.0 & 96.1 & 68.4 & 95.3 & 71.1 \\
 & 50 & 67.2 & 36.7 & 81.3 & 48.4 & 84.4 & 55.9 \\
 & 100 & 26.6 & 14.1 & 64.1 & 38.7 & 71.1 & 46.9 \\
 & 140 & 25.8 & 12.9 & 54.7 & 35.5 & 66.4 & 47.7 \\
\addlinespace[3pt]
Reacher & 25 & 92.2 & 93.0 & 87.5 & 88.3 & 91.4 & 91.8 \\
 & 50 & 97.7 & 97.7 & 94.5 & 93.8 & 97.7 & 96.5 \\
 & 100 & 95.3 & 94.9 & 97.7 & 96.9 & 97.7 & 96.1 \\
 & 150 & 90.6 & 87.5 & 95.3 & 96.9 & 96.9 & 95.7 \\
\addlinespace[3pt]
TwoRoom & 25 & 100.0 & 99.2 & 100.0 & 99.6 & 100.0 & 100.0 \\
 & 50 & 94.5 & 94.1 & 100.0 & 99.6 & 100.0 & 100.0 \\
 & 100 & 89.1 & 88.7 & 96.9 & 98.4 & 98.4 & 100.0 \\
\bottomrule
\end{tabular}
\endgroup
\end{table}

%% file: tables_compact/work.tex
\begin{table}[H]
\centering
\caption{\textbf{Mean episode prediction work at the original offsets.} Values are thousands of predicted five-action blocks per episode. Columns marked Std. show standard starts, and those marked Pert. show the average of each query's perturbed starts.}
\label{tab:work-summary}
\begingroup\small
\setlength{\tabcolsep}{3.5pt}
\begin{tabular}{@{}lrrrrrrrr@{}}
\toprule
Controller & \multicolumn{2}{c}{Cube} & \multicolumn{2}{c}{PushT} & \multicolumn{2}{c}{Reacher} & \multicolumn{2}{c}{TwoRoom} \\
\cmidrule(lr){2-3}\cmidrule(lr){4-5}\cmidrule(lr){6-7}\cmidrule(lr){8-9}
 & \multicolumn{1}{c}{Std.} & \multicolumn{1}{c}{Pert.} & \multicolumn{1}{c}{Std.} & \multicolumn{1}{c}{Pert.} & \multicolumn{1}{c}{Std.} & \multicolumn{1}{c}{Pert.} & \multicolumn{1}{c}{Std.} & \multicolumn{1}{c@{}}{Pert.} \\
\midrule
LeWM & 235.20 & 241.88 & 514.69 & 505.37 & 255.59 & 244.51 & 320.98 & 321.50 \\
CEM (final goal) & 252.62 & 263.17 & 493.56 & 491.10 & 407.76 & 425.31 & 354.95 & 354.95 \\
CEM (learned) & 142.27 & 191.54 & 375.06 & 387.40 & 393.91 & 379.98 & 250.65 & 247.94 \\
AP-CEM & 170.26 & 206.84 & 326.11 & 338.63 & 323.58 & 313.17 & 262.61 & 265.07 \\
CEM (transported) & 193.40 & 223.26 & 347.84 & 393.80 & 306.56 & 310.50 & 289.19 & 291.37 \\
Direct & 0.00 & 0.00 & 0.00 & 0.00 & 0.00 & 0.00 & 0.00 & 0.00 \\
Final-goal rank & 0.12 & 0.16 & 0.38 & 0.42 & 0.20 & 0.20 & 0.15 & 0.16 \\
Learned-target rank & 0.08 & 0.15 & 0.33 & 0.37 & 0.24 & 0.23 & 0.14 & 0.14 \\
AP-rank & 0.10 & 0.15 & 0.30 & 0.34 & 0.24 & 0.24 & 0.14 & 0.14 \\
\bottomrule
\end{tabular}
\endgroup
\end{table}

%% file: sections/E_rq3.tex
\section{RQ3: Memory Targets and Predictive Selection}
\label{app:rq3}
\subsection{Target offsets, memory size, and retrieval keys}
The target-offset ablation uses $q=z_{s_1+\ell}$ with $\ell\in\{1,3,10\}$ while prediction and execution retain $L=5$. By default, we use $\ell=L$ and retrieve only sources with a valid successor at the chosen offset. Memory-size ablations use nested episode subsets of 10, 25, and 50 percent, sampled once with a fixed seed, and recompute retrieval standard deviations in each subset. Table~\ref{tab:ablation-full} gives task-level results for both action rules.

The retrieval-key ablations omit either the distant endpoint or the displacement component. The fixed-span ablation uses $h=H$ throughout execution. Other retrieval and control settings remain unchanged, and Table~\ref{tab:key-full} reports the results.

\input{tables_compact/ablation.tex}
\input{tables_compact/keys.tex}

\subsection{Absolute endpoints and transported displacements}
\label{app:anchor-transport}
The anchoring comparison uses observation-only CEM on the main queries. For a live encoding $z_t$, retrieved source encoding $z_s$, and successor encoding $z_{s+L}$, the two target rules are
\begin{equation}
q_{\rm obs}=z_{s+L},\qquad
q_{\rm transport}=z_t+z_{s+L}-z_s,\qquad L=5.
\label{eq:transported-target}
\end{equation}
The live, source, and successor images use the same frozen encoder. Both rules construct targets from ordered observations and use a common Gaussian prior to generate actions. Their difference is $q_{\rm obs}-q_{\rm transport}=z_s-z_t$: one pursues the recorded endpoint, while the other matches the recorded displacement from the live encoding. The targets coincide when the source and live encodings coincide.

\paragraph{How the target changes action ordering.}
Suppose the current encoding has already moved partway along the recorded displacement. Anchoring requests only the remaining movement to the recorded endpoint. Transport repeats the full displacement from the new state, placing its target beyond that endpoint.

Let $e=z_t-z_s$ denote the current offset from the recorded source. Applying Equation~\eqref{eq:target-order} to the two targets gives
\begin{equation}
\Delta_{q_{\rm obs}}(u,v)-\Delta_{q_{\rm transport}}(u,v)
=2(p_u-p_v)^\top e.
\label{eq:anchor-order}
\end{equation}
When the live state moves away from the recorded source, Equation~\eqref{eq:anchor-order} gives the resulting change in the cost difference between the two actions. Its sign determines which action gains relative to the other. Perturbing the start changes the live state on which both rules act. Within each matched start, the target variants share their initial retrieved record.

\begin{table}[!htbp]
\centering
\caption{\textbf{Anchoring and displacement transport on the main queries.} Under the same CEM search, we compare success (\%) when using the observed endpoint or transported motion. At perturbed starts, we average the two starts of each query, and the mean weights all tasks equally.}
\label{tab:v2-transport}
\input{tables_v2/transport.tex}
\end{table}

\subsection{Exact endpoints and the first action block}
\label{app:v2-exact-endpoints}

At each assigned start, the simulator executes the eight retrieved action
blocks from the same physical state. The frozen encoder maps their
resulting observations to exact candidate endpoints.
For endpoint error and selection regret, we include starts for which all eight
candidates execute five actions. To obtain endpoint error, we first average over candidates.
At perturbed starts, we average the eligible starts of each query and then
give queries equal weight.

The first-block intervention compares Direct, the frozen predictor, simulator
endpoints, and transported displacement estimates. We score predictor and simulator
endpoints against each of the final, observed, and learned targets, whereas
displacement uses the observed target.
After the selected first block, every rollout uses Direct.
When computing success rates, we include all 128 assigned queries per task,
including starts resolved before a policy decision.

Table~\ref{tab:confirmation-diagnostics} reports endpoint error and regret for standard starts and the perturbation average. Table~\ref{tab:v2-first-block} gives every first-block selector with its common continuation.

\input{tables_compact/first_block.tex}

%% file: tables_compact/ablation.tex
\begin{table}[!htbp]
\centering
\caption{\textbf{Target offsets and memory size.} On the main queries, we measure success (\%) while varying the target offset or memory size. By default, the target is the five-step successor and retrieval uses the full memory. Prediction and execution retain five-action blocks throughout.}
\label{tab:ablation-full}
\begingroup\small
\setlength{\tabcolsep}{4pt}
\begin{tabular}{@{}lrrrrrrrr@{}}
\toprule
Setting & \multicolumn{2}{c}{Cube} & \multicolumn{2}{c}{PushT} & \multicolumn{2}{c}{Reacher} & \multicolumn{2}{c}{TwoRoom} \\
\cmidrule(lr){2-3}\cmidrule(lr){4-5}\cmidrule(lr){6-7}\cmidrule(lr){8-9}
 & \multicolumn{1}{c}{Std.} & \multicolumn{1}{c}{Pert.} & \multicolumn{1}{c}{Std.} & \multicolumn{1}{c}{Pert.} & \multicolumn{1}{c}{Std.} & \multicolumn{1}{c}{Pert.} & \multicolumn{1}{c}{Std.} & \multicolumn{1}{c@{}}{Pert.} \\
\midrule
\multicolumn{9}{l}{\textit{AP-CEM}} \\
Default & 46.9 & 29.7 & 69.5 & 64.8 & 75.8 & 78.9 & 46.9 & 46.9 \\
$\ell=1$ & 43.8 & 31.3 & 34.4 & 29.3 & 50.0 & 44.9 & 43.8 & 46.1 \\
$\ell=3$ & 61.7 & 42.6 & 68.8 & 60.9 & 91.4 & 90.6 & 57.8 & 56.3 \\
$\ell=10$ & 19.5 & 15.6 & 58.6 & 57.4 & 87.5 & 89.8 & 50.8 & 48.0 \\
10\% memory & 39.1 & 25.4 & 59.4 & 58.6 & 75.8 & 76.2 & 52.3 & 48.4 \\
25\% memory & 39.1 & 27.7 & 64.8 & 63.3 & 79.7 & 76.2 & 50.0 & 46.9 \\
50\% memory & 43.8 & 30.1 & 67.2 & 61.3 & 76.6 & 76.2 & 46.9 & 47.7 \\
\addlinespace[5pt]
\multicolumn{9}{l}{\textit{AP-rank}} \\
Default & 73.4 & 45.7 & 66.4 & 47.7 & 96.9 & 95.7 & 98.4 & 100.0 \\
$\ell=1$ & 71.9 & 44.5 & 30.5 & 21.9 & 93.8 & 94.9 & 82.0 & 84.0 \\
$\ell=3$ & 71.1 & 44.9 & 49.2 & 35.5 & 95.3 & 96.1 & 98.4 & 97.3 \\
$\ell=10$ & 71.1 & 44.1 & 59.4 & 34.4 & 96.9 & 95.7 & 99.2 & 99.2 \\
10\% memory & 63.3 & 42.2 & 53.1 & 39.5 & 94.5 & 94.5 & 99.2 & 99.6 \\
25\% memory & 67.2 & 44.1 & 58.6 & 43.0 & 96.1 & 95.7 & 98.4 & 98.8 \\
50\% memory & 73.4 & 45.3 & 61.7 & 43.4 & 93.8 & 95.7 & 99.2 & 99.6 \\
\bottomrule
\end{tabular}
\endgroup
\end{table}

%% file: tables_compact/keys.tex
\begin{table}[!htbp]
\centering
\caption{\textbf{Retrieval keys and temporal span.} For AP-rank, we compare success (\%) on the main queries. When omitting a key component, we keep the other components unchanged. Fixing $h=H$ replaces the adaptive span $\max(5,H-t)$.}
\label{tab:key-full}
\begingroup\small
\setlength{\tabcolsep}{3.5pt}
\begin{tabular}{@{}lrrrrrrrr@{}}
\toprule
Setting & \multicolumn{2}{c}{Cube} & \multicolumn{2}{c}{PushT} & \multicolumn{2}{c}{Reacher} & \multicolumn{2}{c}{TwoRoom} \\
\cmidrule(lr){2-3}\cmidrule(lr){4-5}\cmidrule(lr){6-7}\cmidrule(lr){8-9}
 & \multicolumn{1}{c}{Std.} & \multicolumn{1}{c}{Pert.} & \multicolumn{1}{c}{Std.} & \multicolumn{1}{c}{Pert.} & \multicolumn{1}{c}{Std.} & \multicolumn{1}{c}{Pert.} & \multicolumn{1}{c}{Std.} & \multicolumn{1}{c@{}}{Pert.} \\
\midrule
Full key & 73.4 & 45.7 & 66.4 & 47.7 & 96.9 & 95.7 & 98.4 & 100.0 \\
Without distant endpoint & 81.3 & 46.5 & 67.2 & 50.4 & 94.5 & 94.5 & 100.0 & 100.0 \\
Without displacement & 73.4 & 45.7 & 63.3 & 44.9 & 96.1 & 97.7 & 99.2 & 100.0 \\
Fixed span $h=H$ & 4.7 & 3.1 & 0.0 & 0.0 & 12.5 & 19.1 & 21.1 & 23.0 \\
\bottomrule
\end{tabular}
\endgroup
\end{table}

%% file: tables_v2/transport.tex
\begingroup
\small
\setlength{\tabcolsep}{8pt}
\begin{tabular}{@{}llrr@{}}
\toprule
Task & Start & \multicolumn{1}{c}{Observed target} & \multicolumn{1}{c@{}}{Transported target} \\
\midrule
Cube & Standard & \textbf{46.9} & 37.5 \\
 & Perturbed & \textbf{29.7} & 23.8 \\
\addlinespace[4pt]
PushT & Standard & \textbf{69.5} & 61.7 \\
 & Perturbed & \textbf{64.8} & 44.5 \\
\addlinespace[4pt]
Reacher & Standard & 75.8 & \textbf{81.3} \\
 & Perturbed & 78.9 & \textbf{80.5} \\
\addlinespace[4pt]
TwoRoom & Standard & \textbf{46.9} & 34.4 \\
 & Perturbed & \textbf{46.9} & 35.2 \\
\midrule
Mean & Standard & \textbf{59.8} & 53.7 \\
 & Perturbed & \textbf{55.1} & 46.0 \\
\bottomrule
\end{tabular}
\endgroup

%% file: tables_compact/first_block.tex
\begin{table}[H]
\centering
\caption{\textbf{Changing only the first action block.} After selecting the first block, we use the same Direct continuation to measure success (\%). Predictor and simulator selectors score the same eight candidates against the indicated target. All 128 assigned queries per task are included.}
\label{tab:v2-first-block}
\begingroup\small
\setlength{\tabcolsep}{3pt}
\begin{tabular}{@{}lrrrrrrrr@{}}
\toprule
First-block selector & \multicolumn{2}{c}{Cube} & \multicolumn{2}{c}{PushT} & \multicolumn{2}{c}{Reacher} & \multicolumn{2}{c}{TwoRoom} \\
\cmidrule(lr){2-3}\cmidrule(lr){4-5}\cmidrule(lr){6-7}\cmidrule(lr){8-9}
 & \multicolumn{1}{c}{Std.} & \multicolumn{1}{c}{Pert.} & \multicolumn{1}{c}{Std.} & \multicolumn{1}{c}{Pert.} & \multicolumn{1}{c}{Std.} & \multicolumn{1}{c}{Pert.} & \multicolumn{1}{c}{Std.} & \multicolumn{1}{c@{}}{Pert.} \\
\midrule
Direct & 82.0 & 47.7 & 51.6 & 18.0 & 99.2 & 96.9 & 99.2 & 100.0 \\
Predictor / Final & 67.2 & 43.0 & 39.1 & 18.4 & 96.1 & 97.3 & 100.0 & 100.0 \\
Simulator / Final & 73.4 & 44.5 & 43.0 & 19.1 & 96.1 & 97.3 & 99.2 & 99.6 \\
Predictor / Learned & 85.9 & 48.8 & 45.3 & 18.8 & 96.9 & 97.7 & 100.0 & 99.6 \\
Simulator / Learned & 86.7 & 49.2 & 42.2 & 22.3 & 96.9 & 97.7 & 100.0 & 99.6 \\
Predictor / Observed & 78.9 & 47.3 & 53.1 & 22.3 & 95.3 & 96.1 & 99.2 & 99.2 \\
Simulator / Observed & 80.5 & 48.8 & 53.9 & 23.8 & 96.9 & 96.1 & 100.0 & 99.2 \\
Displacement / Observed & 78.1 & 48.4 & 50.8 & 23.0 & 98.4 & 95.7 & 99.2 & 99.6 \\
\bottomrule
\end{tabular}
\endgroup
\end{table}

%% file: sections/F_independent.tex
\section{Independent-Query Target Replication}
\label{app:independent}
The supporting study uses 128 validation episodes per task, disjoint from the main queries, with the same task-specific goal offsets, action allowances, frozen models, and physical success criteria. It passes de-normalized actions directly to the environment, without the clipping to action bounds used in the main study. Each target contrast shares this query set and action mapping.

For CEM, we compare final-goal and observed targets at standard starts, using 30 iterations, 300 candidates, 30 elites, and initial standard deviation $1/3$. For ranking, we compare the same targets on eight retrieved action blocks at the standard start and two preassigned five-action perturbations. The goal and memory remain fixed across starts, and we include all assigned outcomes when computing success rates. Table~\ref{tab:independent-targets} collects the comparisons.

\input{tables_compact/replication.tex}

Observed targets improve success over final-goal scoring in every task under both action rules. Table~\ref{tab:main-results} reports the corresponding main-query target comparison.

%% file: tables_compact/replication.tex
\begin{table}[!htbp]
\centering
\caption{\textbf{Target effects on independent supporting queries.} For each action rule, we compare success (\%) with final-goal and observed targets. Ranking is evaluated at standard and perturbed starts, whereas CEM is evaluated at standard starts. Within each query, we average the two assigned perturbed starts.}
\label{tab:independent-targets}
\begingroup\small
\setlength{\tabcolsep}{5pt}
\begin{tabular}{@{}lrrrrrr@{}}
\toprule
Task & \multicolumn{2}{c}{CEM: Standard} & \multicolumn{2}{c}{Ranking: Standard} & \multicolumn{2}{c}{Ranking: Perturbed} \\
\cmidrule(lr){2-3}\cmidrule(lr){4-5}\cmidrule(lr){6-7}
 & \multicolumn{1}{c}{Final} & \multicolumn{1}{c}{Observed} & \multicolumn{1}{c}{Final} & \multicolumn{1}{c}{Observed} & \multicolumn{1}{c}{Final} & \multicolumn{1}{c@{}}{Observed} \\
\midrule
Cube & 8.6 & 39.1 & 64.8 & 71.1 & 41.4 & 43.4 \\
PushT & 1.6 & 81.3 & 25.8 & 78.1 & 9.0 & 49.2 \\
Reacher & 25.0 & 75.8 & 92.2 & 95.3 & 88.7 & 95.7 \\
TwoRoom & 0.8 & 39.1 & 85.2 & 100.0 & 88.3 & 99.6 \\
\bottomrule
\end{tabular}
\endgroup
\end{table}